\documentclass[twoside,11pt]{article}

\usepackage{microtype}
\usepackage{subfigure}
\usepackage{booktabs}
\usepackage{multirow}
\usepackage{amsmath}
\usepackage{mathtools}
\usepackage{amsthm}
\usepackage{amsfonts}
\usepackage[textsize=tiny]{todonotes}
\usepackage{tikz}
\usepackage{xcolor}
\usepackage{comment}
\usepackage[normalem]{ulem}
\usepackage{algorithm}
\usepackage{algorithmic}

\usepackage{jmlr2e}
\usepackage{thmtools}
\usepackage{thm-restate}
\usepackage[capitalize,noabbrev]{cleveref}
\newtheorem{assumption}[theorem]{Assumption}

\definecolor{ao}{rgb}{0.3, .7, 0.0}

\newcommand{\eee}{\color{black}}
\newcommand{\mzz}{\color{black}}
\newcommand{\EMNGDfull}{\textbf{E}nergy \textbf{M}anifold \textbf{N}atural \textbf{G}radient \textbf{D}escent (\textbf{EMNGD})}

\usepackage{lastpage}
\jmlrheading{TBD}{2026}{1-\pageref{LastPage}}{7/26}{}{TBD}{Zhangyong Liang and Huanhuan Gao}
\makeatletter
\def\ps@jmlrtps{\let\@mkboth\@gobbletwo%
\def\@oddhead{\scriptsize Journal of Machine Learning Research (2026) \hfill Submitted 7/26}%
\def\@oddfoot{}%
\def\@evenhead{}\def\@evenfoot{}}
\makeatother

\ShortHeadings{Energy Manifold Natural Gradient Descent}{Liang and Gao}
\firstpageno{1}

\begin{document}

\title{Energy Manifold Natural Gradient Descent: Riemannian Optimization for Neural PDE Solvers}

\author{\name Zhangyong Liang
       \email zyliang1994@tju.edu.cn \\
       \addr National Center for Applied Mathematics \\
       Tianjin University \\
       Tianjin, 300072, China
       \AND \name Huanhuan Gao
       \email gao\_huanhuan@jlu.edu.cn\\
       \addr  School of Mechanical and Aerospace Engineering \\
       Jilin University \\
       Changchun, 130025, China
       }

\editor{My Editor}

\maketitle

\begin{abstract}%
Energy natural gradient descent (ENGD) aligns parameter updates with the curvature of an underlying function-space energy, but existing formulations assume an unconstrained Euclidean parameter domain.
We introduce \EMNGDfull{}, a manifold optimization framework for physics-informed and variational neural PDE solvers whose parameters lie on a Riemannian manifold.
EMNGD restricts the energy-induced quadratic model to feasible tangent directions and uses retractions to preserve parameter constraints throughout optimization.
Under coercivity, we prove that the push-forward of the undamped EMNGD direction is the best feasible approximation to the function-space Newton vector in the energy metric.
We establish coordinate invariance, exact reduction to ENGD in Euclidean space, global first-order convergence with Armijo backtracking, and robustness to inexact tangent solves.
For quadratic residual energies and generalized Gauss--Newton pullbacks, the Woodbury identity transfers the tangent system to sample space without changing the direction.
Nystr\"om approximation provides scalable sample-space solves with controlled direction error and recovers the exact direction after iterative convergence.
On the evaluated neural PDE benchmarks, EMNGD achieves higher accuracy and faster convergence than the compared state-of-the-art baselines.
Woodbury preserves the EMNGD direction, while scalable-solver diagnostics quantify the accuracy--cost trade-off of preconditioning and residual subsampling.
\end{abstract}


\section{Introduction}

Neural PDE solvers parameterize the unknown solution and minimize a PDE-based loss.
Early work used residual minimization with neural networks~\citep{dissanayake1994neural, lagaris1998artificial}.
PINNs minimize strong-form residuals~\citep{raissi2019physics}, while the deep Galerkin method uses a related residual formulation~\citep{sirignano2018dgm}.
The deep Ritz method minimizes a variational energy~\citep{weinan2018deep}.
Other neural PDE solvers include deep BSDE methods, deep splitting methods, and Fourier neural operators~\citep{han2018solving, weinan2017deep, li2021fourier}.
Recent reviews survey the broader field~\citep{beck2020overview, weinan2021algorithms}.

Training neural PDE solvers to high accuracy remains difficult.
Stiff residual losses and poor conditioning can slow first-order optimization~\citep{wang2021understanding, krishnapriyan2021characterizing}.
Loss weighting and adaptive residual sampling address part of the problem~\citep{wang2021understanding, van2022optimally, wang2022and}.
Curricula and related training strategies provide further controls~\citep{lu2021deepxde, nabian2021efficient, daw2022rethinking}.
Other studies examine residual imbalance and training failure modes~\citep{zapf2022investigating, wang2022respecting, wu2023comprehensive}.
Greedy methods, saddle-point formulations, and particle-swarm methods offer alternatives to direct gradient optimization~\citep{hao2021efficient, zeng2022competitive, davi2022pso}.

Second-order methods instead change the geometry of the update.
Energy natural gradient descent (ENGD) pulls function-space energy curvature back to parameter space~\citep{muller2023achieving, mullerposition}.
Related PDE-constrained methods use mass or stiffness matrices as function-space Gramians~\citep{schwedes2016iteration, schwedes2017mesh}.
Sobolev, Fisher--Rao, and Wasserstein natural gradients have also been studied for PINNs~\citep{nurbekyan2022efficient}.
Gauss--Newton natural gradients and Kronecker-factored curvature provide further approximations~\citep{jnini2024gauss, dangel2024kronecker}.

For quadratic residual energies, the Woodbury identity moves the linear solve from parameter space to sample space.
MinSR uses an analogous sample-space construction in variational Monte Carlo~\citep{chen2023efficient, rende2024simple}.
SPRING momentum and randomized Nystr\"om sketches have been adapted to PINNs~\citep{goldshlager2024kaczmarz, frangella2023randomized}.
Classical Nystr\"om methods construct low-rank positive-semidefinite kernel approximations~\citep{gittens2016nystrom}.
Recent work extends Nystr\"om constructions to Riemannian manifolds~\citep{nie2026nystrom}.

Existing ENGD formulations assume an unconstrained Euclidean parameter domain.
Some neural PDE models impose parameter constraints whose feasible values form a Riemannian manifold $\mathcal M$.
At $x\in\mathcal M$, the realization map sends allowable tangent directions into function space:
\[
T_x\mathcal M \xrightarrow{\ dP_x\ } dP_x(T_x\mathcal M)\subset X.
\]

The energy Hessian defines the curvature of the resulting function-space changes.
An ambient ENGD step followed by projection does not generally minimize the constrained quadratic model.
For an ambient curvature operator $A_x$, gradient $g_x$, and tangent projector $\Pi_x$, one generally has
\begin{equation}\label{eq:intro_projection_noncommutation}
\Pi_x A_x^{-1}g_x \neq
\left(\left.\Pi_x A_x\Pi_x\right|_{T_x\mathcal M}\right)^{-1}\Pi_xg_x.
\end{equation}

An external projection restores feasibility but can change the minimizer of the tangent quadratic model.
Penalty formulations also change the PDE energy and energy curvature.

The mismatch raises a natural question: \textbf{How can an energy natural-gradient method respect a parameter manifold without changing the PDE energy?}

To answer the question, we propose \EMNGDfull{}.
Figure~\ref{fig:emngd_schematic} depicts EMNGD on the energy landscape over $\mathcal M$.
The left view traces feasible iterates from $\theta_0$ toward the low-energy solution $\theta^\ast$, while the inset shows the local update at $x_k$.
The energy-metric solve produces $\eta_{x_k}\in T_{x_k}\mathcal M$, and $R_{x_k}$ maps the tangent point $\widetilde{x}_{k+1}=x_k-\alpha_k\eta_{x_k}$ to the feasible iterate $x_{k+1}$.
The construction separates parameter constraints from function-space energy geometry.
\begin{figure}[t]
    \centering
    \includegraphics[width=0.92\linewidth]{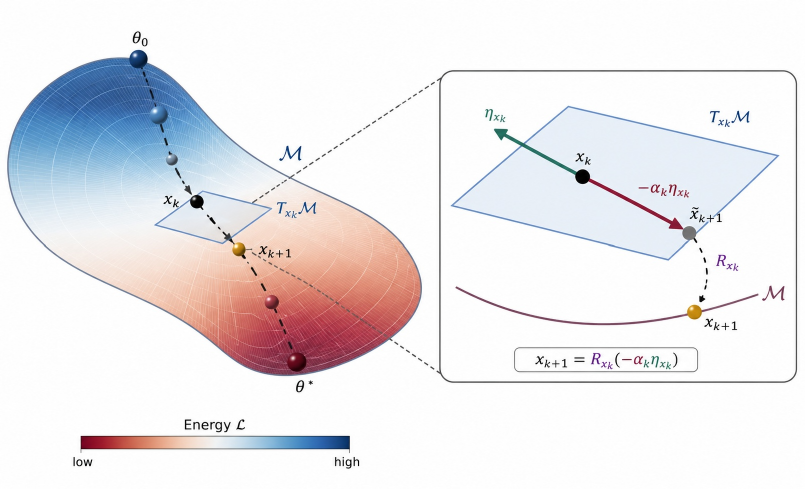}
    \vspace{-12pt}
    \caption{Schematic of an EMNGD update on a parameter manifold. A tangent step $-\alpha_k\eta_{x_k}$ at $x_k$ is retracted to the feasible iterate $x_{k+1}$.}
    \label{fig:emngd_schematic}
\end{figure}
\paragraph{Contributions.}
The contributions are as follows:
\begin{itemize}
    \item \textbf{Intrinsic energy manifold geometry.}
    EMNGD extends ENGD from an unconstrained Euclidean parameter domain to a constrained Riemannian parameter manifold, aligning the parameter geometry with PDE residual constraints.
    The energy-induced quadratic model is defined directly over feasible tangent directions, while retraction-based updates preserve the parameter constraints.
    The intrinsic construction retains the original residual energy and differs from post-hoc projection of an ambient ENGD step.

    \item \textbf{Best-admissible Newton correction.}
    The main theoretical result characterizes EMNGD as the best admissible approximation to the function-space Newton correction under the energy metric.
    The admissible correction is restricted jointly by the neural realization map and the tangent space of the parameter manifold.
    For quadratic energies, the natural-gradient vector represents the projected current solution error, while the negative retracted step moves toward the corresponding projected solution correction.

    \item \textbf{Geometric consistency and convergence.}
    Positive damping makes the tangent energy metric positive definite and yields a unique EMNGD direction.
    Coordinate invariance and exact reduction to ENGD in Euclidean space establish consistency across parameter representations.
    Under the stated metric-equivalence and retraction-smoothness assumptions, Armijo backtracking yields global first-order convergence.
    Controlled inexact tangent solves also preserve the descent property.

    \item \textbf{Tangent-space scalable solvers.}
    For quadratic residual energies and generalized Gauss--Newton pullbacks, exact Woodbury duality transfers the EMNGD tangent system to sample space without changing the damped direction.
    Nystr\"om sketch-and-solve provides a low-rank approximate direction, while Nystr\"om preconditioning recovers the exact direction after iterative convergence.
    An explicit error bound connects kernel approximation quality and damping with the accuracy of the computed tangent direction.

    \item \textbf{Scalability with direction control.}
    Numerical studies verify the Euclidean reduction and the primal--dual agreement of the Woodbury implementation.
    Large-sample diagnostics quantify the effects of sketch rank, damping, and residual subsampling on direction error, convergence, memory consumption, and runtime.
    The results identify the regimes in which scalable solvers retain EMNGD accuracy and the regimes in which approximation or sampling error becomes dominant.
    
\end{itemize}

\paragraph{Notation.}
We denote the space of $p$-integrable functions on \(\Omega\subseteq\mathbb R^d\) by \(L^p(\Omega)\) and use the canonical norm of \(L^p(\Omega)\). 
For a sufficiently smooth function $u$, let $\partial_i u = \partial u / \partial x_i$.  Let $(D^lu)_{i_1, \dots, i_l} \coloneqq \partial_{i_1}\dots\partial_{i_l} u$ denote the $l$th-derivative tensor. 
Let $\nabla u = (\partial_1 u, \dots, \partial_d u)^\top$ denote the gradient.  Define the Laplace operator by $\Delta u \coloneqq \sum_{i=1}^d \partial_i^2 u$.
We denote the \emph{Sobolev space} 
of functions with weak derivatives up to order \(k\) in $L^p(\Omega)$ by \(W^{k,p}(\Omega)\), which is a Banach space with the norm
\[ \lVert u\rVert_{W^{k,p}(\Omega)}^p \coloneqq \sum_{l=0}^k\lVert D^{l} u \rVert_{L^p(\Omega)}^p, \]
in the following, we mostly work with the case $p=2$ and write $H^k(\Omega)$ instead of $W^{k,2}(\Omega)$.

Let \(d, m, L, N_0, \dots, N_L\) be natural numbers.  Let $\theta = \left((A_1, b_1), \dots, (A_L, b_L)\right)$, where \(A_l\in\mathbb R^{N_{l}\times N_{l-1}}\), \(b_l\in\mathbb R^{N_l}\), \(N_0 = d\), and \(N_L = m\).  Each pair \((A_l, b_l)\) defines an affine map \(T_l\colon \mathbb R^{N_{l-1}} \to\mathbb R^{N_l}\).  Given an activation function \(\rho\colon\mathbb R\to\mathbb R\), the \emph{neural network function with parameters} \(\theta\) is
\[u_\theta\colon\mathbb R^d\to\mathbb R^m, \quad x\mapsto T_L(\rho(T_{L-1}(\rho(\cdots \rho(T_1(x)))))).\]

The \emph{number of trainable parameters} of such a network is \(\sum_{l=0}^{L-1}(n_{l}+1)n_{l+1}\)
.
We call a network with depth \(2\) \emph{shallow} and a deeper network \emph{deep}. 
In the remainder, we restrict ourselves to the case \(m=1\) since we only consider real-valued functions.
Our experiments use $\tanh$ activations, which are required for the smoothness of $u_\theta$ and the map $\theta\mapsto u_\theta$.
For $A\in\mathbb R^{n\times m}$, we denote any pseudo inverse of $A$ by $A^+$.

\section{Preliminaries}

Various neural solvers for the approximate solution of PDEs have been suggested~\citep{beck2020overview, weinan2021algorithms, kovachki2021neural}.
Neural PDE solvers parameterize an approximate solution and minimize either a residual energy or a variational energy.
The preliminary discussion introduces both formulations, fixes a common function-space setup, and summarizes the optimization motivation for natural gradients.

\paragraph{Residual-form neural PDE solvers.}
Residual-form neural PDE solvers minimize the PDE residual and boundary mismatch.
Consider a general partial differential equation of the form
\begin{align}
    \begin{split}
        \mathcal L u 
        & = f \quad \text{in } \Omega \\
        \mathcal B u & = g \quad \text{on } \partial\Omega,
    \end{split}
    \label{eq:Poisson}
\end{align}
where $\Omega\subseteq\mathbb R^d$ is open, $\mathcal L$ is a possibly nonlinear partial differential operator, and $\mathcal B$ is a boundary-value operator.  We seek $u$ in a Hilbert space $X$.  Assume that $f$ is square integrable on $\Omega$ and $g$ is square integrable on $\partial\Omega$.  Equation~\eqref{eq:Poisson} then admits the minimization formulation
\begin{equation}
    E(u) =  \int_\Omega (\mathcal L u - f)^2
    \mathrm{d}x + \tau \int_{\partial\Omega} (\mathcal B u-g)^2\mathrm ds,
\end{equation}
for a penalization parameter $\tau > 0$.  A function $u\in X$ solves \eqref{eq:Poisson} exactly when $E(u)=0$.  For an approximate solution, parameterize $u_\theta$ by a neural network and minimize the parameters $\theta\in\mathbb R^p$ using
\begin{equation}
    L(\theta) \coloneqq \int_\Omega (\mathcal Lu_\theta - f)^2\mathrm dx + \tau \int_{\partial\Omega}\mathcal (\mathcal Bu_\theta - g)^2\mathrm ds.
\label{eq:loss}
\end{equation}

Residual minimization for neural PDE solvers traces back to~\citep{dissanayake1994neural, lagaris1998artificial}.
The deep Galerkin method and physics-informed neural networks use related residual objectives~\citep{sirignano2018dgm,raissi2019physics}.
Data terms can be added to the loss.  Numerical implementations discretize the integrals with interior and boundary samples.

\paragraph{Variational neural PDE solvers.}
Weak PDE formulations often use an energy functional whose Euler--Lagrange equations recover the weak form.  \citet{ritz1909neue} used the idea to compute polynomial approximation coefficients.  \citet{weinan2018deep} introduced the name \emph{deep Ritz method} for neural networks.
Given a variational energy $E\colon X \to \mathbb R$ on a Hilbert space $X$, parameterize the ansatz by $u_\theta$.  The loss is $L(\theta)\coloneqq E(u_\theta)$. 
For $-\Delta u = f$, the residual energy is $u\mapsto\lVert \Delta u + f \rVert_{L^2(\Omega)}^2$.  The variational energy is $u\mapsto\frac12\lVert \nabla u\rVert_{L^2(\Omega)}^2-\int_\Omega fu\mathrm dx$.  The two energies require different smoothness and belong to different Sobolev spaces.

Essential boundary values enter the deep Ritz method differently from PINNs. 
For PINNs, the unique minimizer is the PDE solution for every $\tau>0$.  In the deep Ritz method, the penalized minimizer solves a perturbed Robin problem.  Accurate approximation of the original problem requires large penalty parameters, which cause ill-conditioning~\citep{muller2022error, courte2023robin}.

\paragraph{Function-space setup.}
Residual and variational formulations minimize an energy $E\colon X\to\mathbb R$.  The parameter loss is $L(\theta) \coloneqq E(u_\theta)$.  
Assume that $X$ is a Hilbert space, $u_\theta\in X$, and $E$ has a unique minimizer $u^*\in X$. 
And assume that $P\colon\mathbb R^p\to X$, $\theta\mapsto u_\theta$, is differentiable.  Define $\mathcal F_\Theta=\{u_\theta:\theta\in\mathbb R^p\}$.  The generalized tangent space of $\mathcal F_\Theta$ is
\begin{equation}
    T_\theta \mathcal F_\Theta \coloneqq  
    \operatorname{span} \left\{\partial_{\theta_i} u_\theta : i=1, \dots, p \right\}.
\end{equation}

\paragraph{Optimization challenge.}
First-order optimization can stagnate on residual losses, even for simple PDEs.  Residual stiffness contributes to poor conditioning~\citep{wang2021understanding}.  Squared residuals can further increase the condition number~\citep{zeng2022competitive}.  Poor conditioning slows iterative solvers such as gradient descent.
Figure~\ref{fig:problem_1d} illustrates the challenge in one dimension.  Across Poisson, heat, and nonlinear equations, SGD, Adam, BFGS, L-BFGS, and Adam--L-BFGS either stagnate at large relative $L^2$ errors or need many iterations.

\begin{figure}[t]
    \centering
    \includegraphics[width=\linewidth]{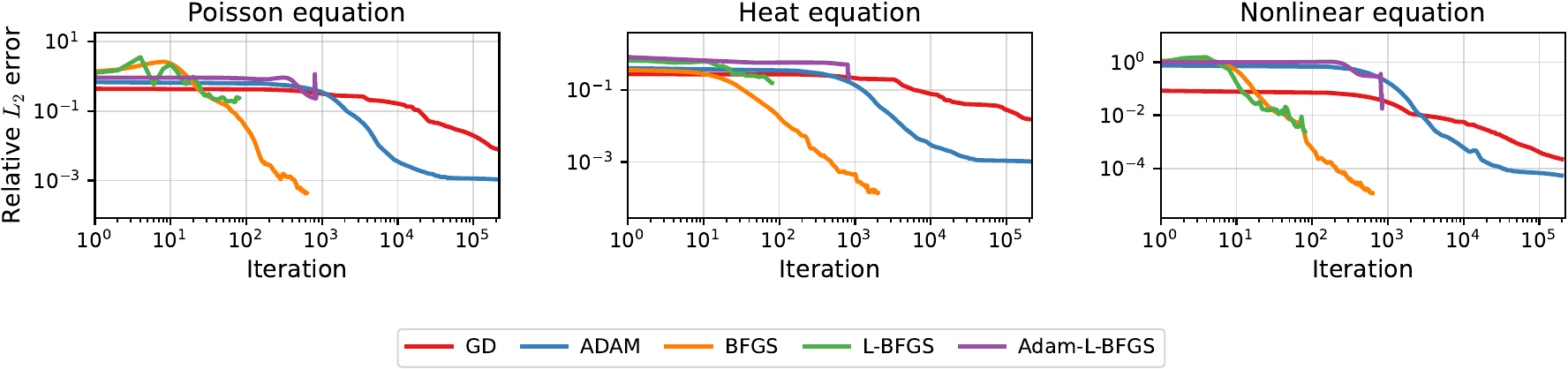}
    \caption{Relative $L^2$ errors of standard optimizers on one-dimensional PDE benchmarks.}
    \label{fig:problem_1d}
\end{figure}

\paragraph{Natural Gradient Descent.}
\citet{amari1996neural} originally proposed natural gradient descent (NGD) for Euclidean parameter optimization.
Given a positive-definite metric $G(\theta)$, NGD solves $G(\theta)d_\theta=\nabla L(\theta)$ and takes $-d_\theta$ as the descent direction.
In statistical models, $G(\theta)$ is usually the Fisher information matrix.
The metric accounts for local model geometry and yields a coordinate-invariant direction when $G(\theta)$ transforms as a pullback metric.
\citet{muller2023achieving} introduced ENGD for neural PDE solvers by replacing the Fisher metric with an energy-induced metric.
For $u_\theta=P(\theta)$, the energy Gram matrix has entries $G_E(\theta)_{ij}=D^2E(u_\theta)(\partial_{\theta_i}u_\theta,\partial_{\theta_j}u_\theta)$.
The damped system $(G_E(\theta)+\lambda I)d=\nabla L(\theta)$ defines the ENGD direction.
The energy metric gives the direction a direct function-space interpretation.  The following sections extend the construction to constrained parameter manifolds.


\paragraph{Manifold Optimization.}
Let $(\mathcal M,g^0)$ be a smooth Riemannian manifold and consider the minimization of a differentiable objective $F\colon\mathcal M\to\mathbb R$.
At $x\in\mathcal M$, the tangent space $T_x\mathcal M$ contains the feasible local directions.  The metric $g_x^0$ defines an inner product on that tangent space.
The Riemannian gradient $\operatorname{grad}^{0}F(x)\in T_x\mathcal M$ is defined by
\begin{equation}
    g_x^0(\operatorname{grad}^{0}F(x),\xi)=dF_x[\xi]\qquad \text{for all } \xi\in T_x\mathcal M.
\end{equation}

A retraction $R_x\colon T_x\mathcal M\to\mathcal M$ maps a tangent vector back to the manifold.  A Riemannian optimization step first computes $\eta_x\in T_x\mathcal M$ and then sets
\begin{equation}
    x_{k+1}=R_{x_k}(-\alpha_k\eta_{x_k}).
\end{equation}

Manifold optimization is useful when parameters satisfy hard constraints.  The realization map sends a tangent direction to a first-order change in function space.

\section{Energy manifold natural gradient descent (EMNGD)}\label{sec:natgrad}
We next develop the geometric formulation of energy natural gradient descent.
Classical ENGD uses function-space energy curvature on the tangent space of the current neural model.
The Euclidean formulation has several limitations.  The curvature system can be expensive and ill-conditioned.  The Newton interpretation is local to the current tangent space.  Damping, pseudoinverses, or least-squares solves are often needed for numerical stability.
Existing scalable ENGD variants only change how that system is solved or approximated.  Examples include kernel, dual, randomized, and low-rank linear algebra.  Such variants do not encode hard constraints, quotient symmetries, or other feasible-set geometries.
The manifold formulation begins with the tangent space $T_x\mathcal M$.  The pullback energy Hessian acts on that space, and a retraction follows each step.
\citet{amari1996neural} popularized natural gradients for parameter estimation in supervised learning and blind source separation.
Natural gradients use a chosen metric to define update directions, including Fisher, product-Fisher, Wasserstein, and Sobolev geometries~\citep{kakade2001natural, li2018natural, nurbekyan2022efficient}, and have been applied to reinforcement learning~\citep{kakade2001natural, peters2003reinforcement, bagnell2003covariant, morimura2008new}, inverse problems~\citep{nurbekyan2022efficient}, neural-network training~\citep{schraudolph2002fast, pascanu2014revisiting, martens2020new}, and generative models~\citep{shen2020sinkhorn, lin2021wasserstein}.
A key issue for natural gradients is the choice of function-space geometry.  The geometry can be defined axiomatically or through the Hessian of a potential function~\citep{amari2010information, amari2016information, wang2022hessian, Mueller2022Convergence}. 
EMNGD uses the exact function-space Hessian when that Hessian is positive semidefinite on realized directions.
For residual objectives, the implementation may instead use the generalized Gauss--Newton (GGN) curvature.

For $E(u)=\tfrac12\|\mathcal Q(u)\|^2$, the exact Hessian is
\[
D^2E(u)[v,w]=\langle D\mathcal Q(u)v,D\mathcal Q(u)w\rangle
+\langle \mathcal Q(u),D^2\mathcal Q(u)[v,w]\rangle.
\]

The GGN retains the first term.  The approximation equals the exact Hessian when $\mathcal Q$ is affine.  For nonlinear $\mathcal Q$, the approximation discards the residual-weighted second derivative.
Related curvature methods have been proposed for supervised neural-network training~\citep{ren2019efficient, cai2019gram, gargiani2020promise, martens2020new}.
Our applications may involve infinite-dimensional or non-strongly-convex objectives.

\begin{assumption}[Geometric and analytic setting]\label{ass:basic_emngd}
Let $(\mathcal M,g^0)$ be a finite-dimensional smooth Riemannian manifold.  
And let $R$ be a retraction, i.e.,
\[
    R_x(0_x)=x,\qquad DR_x(0_x)=\mathrm{id}_{T_x\mathcal M}.
\]

Let $X$ be a real Hilbert space, $E:X\to\mathbb R$ be twice Fr\'echet differentiable, and $P:\mathcal M\to X$ be twice differentiable.
We minimize $F=E\circ P$ on $\mathcal M$.  The differential $J_x=dP_x:T_x\mathcal M\to X$ maps a tangent direction to function space.
The baseline Riemannian gradient is defined by
\[
    g_x^0(\operatorname{grad}^0F(x),\zeta)=dF_x[\zeta],
    \qquad \zeta\in T_x\mathcal M.
\]
\end{assumption}

The parameter manifold $\mathcal M$ and the image $J_x(T_x\mathcal M)\subset X$ have different roles.
The parameter manifold defines the allowed directions.  The image contains the corresponding first-order changes in function space.
The energy $E\colon X\to\mathbb R$ is twice differentiable.  The setting covers both PINNs and the deep Ritz method.
The energy Hessian induces on each tangent space the pullback bilinear form
\begin{equation}\label{eq:pullback_energy_metric}
    g_x^E(\xi,\zeta)\coloneqq D^2E(P(x))(J_x\xi,J_x\zeta),\qquad \xi,\zeta\in T_x\mathcal M.
\end{equation}

If $g_x^E$ is positive definite, the bilinear form defines a Riemannian metric on $\mathcal M$.  If $g_x^E$ is not positive definite, we use the damped form
\begin{equation}\label{eq:damped_energy_metric}
    g_x^{E,\lambda}(\xi,\zeta)\coloneqq g_x^E(\xi,\zeta)+\lambda g_x^0(\xi,\zeta),\qquad \lambda\geq 0,
\end{equation}
for $\lambda>0$, damping gives a regularized approximation to the minimum-norm pseudoinverse solution.
The least-squares implementations use the same regularized tangent system.
The EMNGD direction is the tangent vector $\eta_x\in T_x\mathcal M$ satisfying
\begin{equation}\label{eq:manifold_emngd_direction}
    g_x^{E,\lambda}(\eta_x,\zeta)=dF_x[\zeta]
    =g_x^0(\operatorname{grad}^{0}F(x),\zeta)
    \qquad\text{for all }\zeta\in T_x\mathcal M.
\end{equation}

The equation is a linear system on the tangent space.  EMNGD then updates by
\begin{equation}\label{eq:def_emng_update}
    x_{k+1}=R_{x_k}(-\alpha_k\eta_{x_k}).
\end{equation}

In local coordinates $x=\phi(\xi)$, the Euclidean energy Gram matrix becomes the pullback $\widetilde G_E(\xi)=J_\phi(\xi)^\top G_E(\phi(\xi))J_\phi(\xi)$.
The unconstrained parameter case used in our experiments corresponds to $\mathcal M=\mathbb R^p$, $g^0$ equal to the Euclidean metric, and the retraction $R_\theta(v)=\theta+v$.

\begin{proposition}[Operator form and variational characterization]\label{prop:operator_variational}
Assume that $D^2E(P(x))$ is symmetric positive semidefinite as a bilinear form on $X$.
Then $g_x^E$ is symmetric positive semidefinite on $T_x\mathcal M$.
If $\lambda>0$, then $g_x^{E,\lambda}$ is positive definite and there exists a unique $g_x^0$-self-adjoint positive definite operator
\[
    A_x^\lambda:T_x\mathcal M\to T_x\mathcal M
\]
such that
\begin{equation}\label{eq:operator_definition}
    g_x^0(A_x^\lambda\xi,\zeta)=g_x^{E,\lambda}(\xi,\zeta)
    \qquad\text{for all }\xi,\zeta\in T_x\mathcal M.
\end{equation}

The EMNGD direction is uniquely
\begin{equation}\label{eq:operator_direction}
    \eta_x=(A_x^\lambda)^{-1}\operatorname{grad}^0F(x),
\end{equation}
which is the unique minimizer of
\begin{equation}\label{eq:emngd_quadratic_functional}
    Q_x(\xi)=\frac12 g_x^{E,\lambda}(\xi,\xi)-dF_x[\xi],
    \qquad \xi\in T_x\mathcal M.
\end{equation}
\end{proposition}
\begin{proof}
For any $\xi\in T_x\mathcal M$, $g_x^E(\xi,\xi)=D^2E(P(x))[J_x\xi,J_x\xi]\geq0$.
If $\lambda>0$ and $\xi\neq0$, then
$g_x^{E,\lambda}(\xi,\xi)\geq\lambda g_x^0(\xi,\xi)>0$.
The Riesz representation theorem on the finite-dimensional inner-product space $(T_x\mathcal M,g_x^0)$ gives $A_x^\lambda$, and symmetry of $g_x^{E,\lambda}$ makes $A_x^\lambda$ self-adjoint.
Substituting~\eqref{eq:operator_definition} into~\eqref{eq:manifold_emngd_direction} gives $A_x^\lambda\eta_x=\operatorname{grad}^0F(x)$.
The first-order optimality condition for $Q_x$ is exactly~\eqref{eq:manifold_emngd_direction}, and strict convexity follows from positive definiteness.
\end{proof}

\begin{theorem}[Coordinate form and coordinate invariance]\label{thm:coordinate_invariance}
Let $\phi:U\subset\mathbb R^q\to\mathcal M$ be a local chart with $x=\phi(y)$ and $e_i=\partial_i\phi(y)$.
Writing $\eta_x=\sum_{i=1}^q v_i e_i$, define
\[
    (G_E^\phi)_{ij}=g_x^E(e_i,e_j),\qquad
    (G_0^\phi)_{ij}=g_x^0(e_i,e_j),\qquad
    b_i=dF_x[e_i].
\]

Then the coordinate vector $v$ satisfies
\begin{equation}\label{eq:coordinate_linear_system}
    (G_E^\phi+\lambda G_0^\phi)v=b.
\end{equation}

For $\lambda>0$, the tangent system has a unique solution and the tangent vector $\eta_x$ is independent of the chosen chart.
For $\mathcal M=\mathbb R^p$ with the Euclidean metric, $P(\theta)=u_\theta$, and $R_\theta(v)=\theta+v$,~\eqref{eq:coordinate_linear_system} becomes
\begin{equation}\label{eq:euclidean_engd_reduction}
    (G_E(\theta)+\lambda I)d=\nabla L(\theta).
\end{equation}

For $\lambda=0$, the Moore--Penrose convention gives the minimum-norm pseudoinverse direction $d=G_E(\theta)^+\nabla L(\theta)$ when the undamped system is singular.
\end{theorem}
\begin{proof}
Testing~\eqref{eq:manifold_emngd_direction} with each basis vector $e_j$ gives~\eqref{eq:coordinate_linear_system}.
For $\lambda>0$, the coefficient matrix is positive definite because
\begin{equation}
v^\top(G_E^\phi+\lambda G_0^\phi)v
    =g_x^{E,\lambda}\Big(\sum_i v_i e_i,\sum_i v_i e_i\Big)>0,
\end{equation}
for every nonzero $v$.
The vector reconstructed from the coordinate solution satisfies the weak equation~\eqref{eq:manifold_emngd_direction}; uniqueness of that equation implies chart independence.
The Euclidean reduction follows from the identity chart, for which $G_0=I$ and $b=\nabla L(\theta)$.
\end{proof}

\begin{figure}[t]
    \centering
    \includegraphics[width=\linewidth]{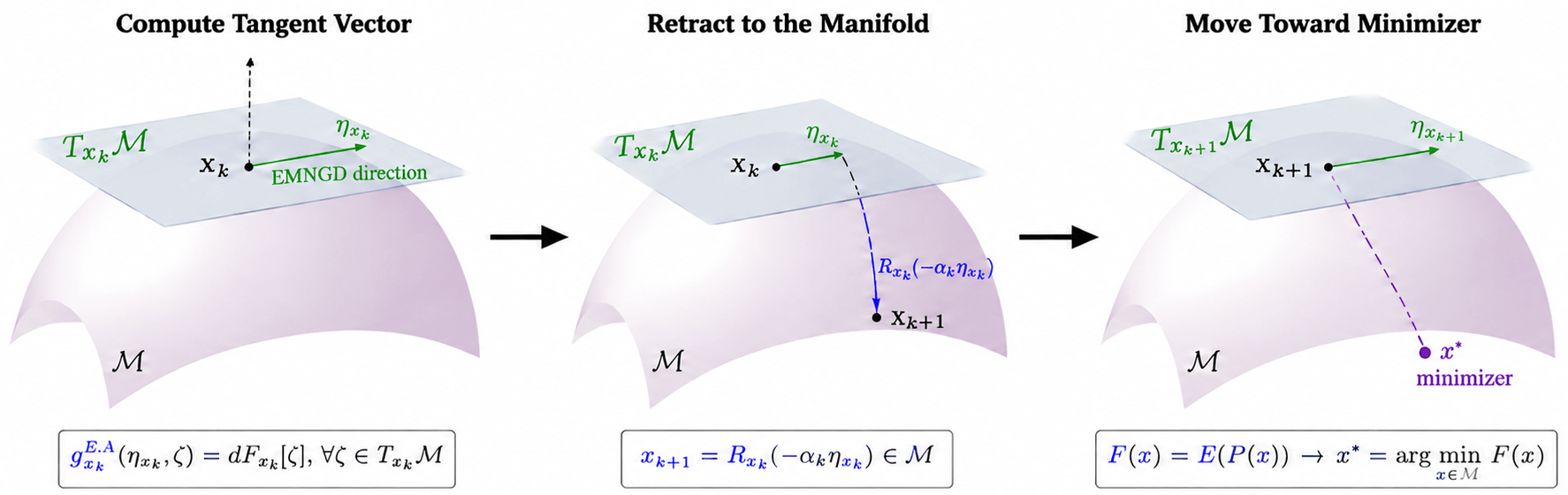}
    \caption{EMNGD update on a parameter manifold. From left to right, EMNGD solves the energy-metric equation in $T_{x_k}\mathcal M$, retracts $-\alpha_k\eta_{x_k}$ to $x_{k+1}\in\mathcal M$, and repeats the update toward $x^\ast$.}
    \label{fig:emngd_idea}
\end{figure}
We define the \emph{Hilbert} and \emph{energy Gram matrices} by 
\begin{align}\label{eq:gramHilbert}
    G_H(\theta)_{ij} \coloneqq \langle\partial_{\theta_i} u_\theta, \partial_{\theta_j} u_\theta\rangle_X, \quad 
\end{align}
and
\begin{equation}\label{eq:gramEnergy}
    G_E(\theta)_{ij} \coloneqq D^2E(u_\theta)(\partial_{\theta_i} u_\theta, \partial_{\theta_j} u_\theta).
\end{equation}

The \emph{Hilbert natural-gradient direction} $\nabla^H L(\theta) = G_H(\theta)^+\nabla L(\theta)$ \mzz uses the Sobolev inner product $\langle \cdot, \cdot\rangle_{X}$ for neural-network training~\citep{nurbekyan2022efficient}.
For a Sobolev space $X$, the direction is also called the \emph{Sobolev natural gradient}; H-NG denotes the \emph{Hilbert natural gradient}\eee.
Natural-gradient theory establishes\footnote{For regular and singular Gram matrices and finite-dimensional spaces, see~\citep{amari2016information, van2022invariance}.  The appendix gives an argument for infinite-dimensional spaces.} that
\footnote{\mzz Here, the Hilbert space gradient $\nabla E(u)\in X$ is the unique element satisfying $\langle \nabla E(u), v\rangle_X = DE(u)v$, where $DE$ denotes the Fréchet derivative.
\eee
}
\begin{equation}
    DP_\theta\nabla^HL(\theta) = \Pi_{T_\theta \mathcal F_\Theta}( \nabla E(u_\theta)).
\end{equation}

In words, following the natural gradient amounts to moving along the projection of the Hilbert space gradient onto the model's tangent space in function space.
The observation that identifying the function space gradient via the Hessian leads to a Newton update motivates the concept of energy natural gradients that we now introduce.

\begin{definition}[Energy Manifold Natural Gradient]\label{def:ENG}
Under Assumption~\ref{ass:basic_emngd}, the \emph{energy manifold natural gradient} at $x\in\mathcal M$ is the tangent vector $\eta_x$ solving~\eqref{eq:manifold_emngd_direction}.
The associated descent direction is $-\eta_x$, and the algorithmic update is the retracted step~\eqref{eq:def_emng_update}.
In the Euclidean parameter case $\mathcal M=\mathbb R^p$, $R_\theta(v)=\theta+v$, and $\lambda=0$, Definition~\ref{def:ENG} reduces to
\begin{equation}
    \nabla^E L(\theta) \coloneqq G_E(\theta)^+\nabla L(\theta),
\end{equation}
the usual energy natural gradient direction.
\end{definition}

\subsection{Scalable solvers for the EMNGD tangent system}

For a linear PDE operator $\mathcal L$, the residual 
yields a quadratic energy, and the energy Gram matrix takes the form

\begin{align}
    \begin{split}
        G_E(\theta)_{ij} 
        &= 
        \int_\Omega \mathcal L (\partial_{\theta_i}u_\theta) \mathcal L (\partial_{\theta_j}u_\theta) \mathrm dx 
        \\
        &+ 
        \tau \int_{\partial\Omega} \mathcal B (\partial_{\theta_i}u_\theta) \mathcal B (\partial_{\theta_j}u_\theta)  \mathrm ds
    \end{split}
\end{align}

The residual-energy expression also exposes the low-rank structure used by scalable implementations.
After quadrature, the residual loss can be written as
\begin{equation}
    L(\theta)=\frac12\|r(\theta)\|_2^2,\qquad r(\theta)\in\mathbb R^N,
\end{equation}
where the entries of $r$ collect the weighted interior and boundary residuals.
Let $J(\theta)=D_\theta r(\theta)\in\mathbb R^{N\times p}$ be the residual Jacobian.
For a linear PDE operator, $J^\top J$ is the exact pullback of the quadratic function-space Hessian.
For a nonlinear residual map, $J^\top J$ is the GGN pullback and omits residual-weighted second-derivative terms from the parameter Hessian.
Then
\begin{equation}\label{eq:residual_jacobian_energy}
    G_E(\theta)=J(\theta)^\top J(\theta),\qquad
    \nabla L(\theta)=J(\theta)^\top r(\theta).
\end{equation}

The damped EMNGD direction satisfies
\begin{equation}\label{eq:emngd_damped_primal}
    \nabla^E_\lambda L(\theta)
    =
    \big(J^\top J+\lambda I\big)^{-1}J^\top r,
\end{equation}
where $J=J(\theta)$ and $r=r(\theta)$.
Applying the push-through identity, equivalently the Woodbury matrix identity, gives the sample-space form
\begin{equation}\label{eq:emngd_woodbury}
    \big(J^\top J+\lambda I\big)^{-1}J^\top r
    =
    J^\top\big(JJ^\top+\lambda I\big)^{-1}r.
\end{equation}

For an embedded parameter manifold, let $\mathbf J_x:\mathbb R^p\to\mathbb R^N$ be the ambient residual Jacobian.  Let $\Pi_x:\mathbb R^p\to T_x\mathcal M$ be the orthogonal projector onto the tangent space.
The corresponding intrinsic direction is
\begin{equation}\label{eq:emngd:mopt:manifold-woodbury}
    \eta_x=\Pi_x\mathbf J_x^\top
    \big(\mathbf J_x\Pi_x\mathbf J_x^\top+\lambda I\big)^{-1}r(x).
\end{equation}

The Woodbury identity computes the same EMNGD tangent direction from an $N\times N$ system rather than a $p\times p$ system.
The reduction is useful when quadrature or collocation samples are far fewer than trainable parameters.  The dominant solve then occurs in sample space.
The matrix $JJ^\top$ is the energy analogue of the empirical neural tangent kernel~\citep{jacot2018neural}.  Efficient kernel construction techniques can be used without changing the EMNGD geometry~\citep{novak2022fast}.

\paragraph{Exact Woodbury duality and Nystr\"om sketches.}
Let $\mathcal J_x=D r_x:T_x\mathcal M\to\mathbb R^N$ denote the residual differential, and let $\mathcal J_x^*$ be the adjoint induced by $g_x^0$ and the Euclidean inner product.
The residual gradient is $\operatorname{grad}^0F(x)=\mathcal J_x^*r(x)$.
For the quadratic residual energy or the corresponding GGN metric, the intrinsic damped direction and sample-space kernel are
\begin{equation}\label{eq:intrinsic_residual_direction}
    \eta_x=(\mathcal J_x^*\mathcal J_x+\lambda I_x)^{-1}\mathcal J_x^*r(x),
    \qquad K_x=\mathcal J_x\mathcal J_x^*.
\end{equation}

The push-through identity gives $\eta_x=\mathcal J_x^*(K_x+\lambda I)^{-1}r(x)$.  Woodbury is an exact dual representation of the damped tangent direction.  In embedded coordinates, $\mathcal J_x=\mathbf J_x\Pi_x$, which recovers~\eqref{eq:emngd:mopt:manifold-woodbury}.

\begin{proposition}[Nystr\"om direction error]\label{prop:nystrom_direction_error}
Let $\widehat K_x\succeq0$ be a rank-$\ell$ Nystr\"om approximation of $K_x$, and define
\[
    \widetilde\eta_x=\mathcal J_x^*(\widehat K_x+\lambda I)^{-1}r(x),
    \qquad \lambda>0.
\]

Then
\begin{equation}\label{eq:nystrom_direction_error}
    \|\widetilde\eta_x-\eta_x\|_0
    \leq
    \frac{\|\mathcal J_x\|_{0\to2}\,\|K_x-\widehat K_x\|_2}{\lambda^2}
    \|r(x)\|_2.
\end{equation}
\end{proposition}
\begin{proof}
The resolvent identity gives
\[
 (\widehat K_x+\lambda I)^{-1}-(K_x+\lambda I)^{-1}
 =(\widehat K_x+\lambda I)^{-1}(K_x-\widehat K_x)(K_x+\lambda I)^{-1}.
\]

The positive damping is essential: positive semidefiniteness bounds both inverse norms by $\lambda^{-1}$.
Without $\lambda>0$, the inverse norm can diverge when $K_x$ or $\widehat K_x$ is singular.
Applying $\mathcal J_x^*$ and the operator-norm bound gives~\eqref{eq:nystrom_direction_error}.
\end{proof}

A Riemannian Nystr\"om construction can also approximate the tangent Gram operator $\mathcal G_x=\mathcal J_x^*\mathcal J_x$ directly.  For a rank-$\ell$ tangent sketch $\mathcal P_x$, the intrinsic approximation is
\[
    \widehat{\mathcal G}_x=
    \mathcal G_x\mathcal P_x
    (\mathcal P_x^*\mathcal G_x\mathcal P_x)^\dagger
    \mathcal P_x^*\mathcal G_x.
\]

The intrinsic approximation is $g_x^0$-self-adjoint, positive semidefinite, and has rank at most $\ell$~\citep{nie2026nystrom}.
EMNGD instead sketches the dual kernel $K_x$ after the exact Woodbury transformation.
A sketch-and-solve update uses $\widetilde\eta_x$ and is approximate.  Equation~\eqref{eq:nystrom_direction_error} bounds the resulting direction error.
A Nystr\"om preconditioner changes only the conditioning of an iterative solve of $(K_x+\lambda I)a=r(x)$.
After convergence, $\eta_x=\mathcal J_x^*a$ is the exact Woodbury direction.

Equivalently,~\eqref{eq:emngd_damped_primal} is the solution of the Tikhonov-regularized least-squares problem
\begin{equation}\label{eq:emngd_tikhonov}
    \nabla^E_\lambda L(\theta)
    =
    \arg\min_{\psi\in\mathbb R^p}
    \left\{\frac12\|J\psi-r\|_2^2+\frac{\lambda}{2}\|\psi\|_2^2\right\}.
\end{equation}

The first-order condition of~\eqref{eq:emngd_tikhonov} is
\[
    (J^\top J+\lambda I)\psi=J^\top r,
\]
which gives~\eqref{eq:emngd_damped_primal}; the identity
\[
    (J^\top J+\lambda I)J^\top=J^\top(JJ^\top+\lambda I)
\]
then gives~\eqref{eq:emngd_woodbury}.
If $r\in\operatorname{range}(J)$, singular-value decomposition shows that $\nabla^E_\lambda L(\theta)\to J^+r$ as $\lambda\downarrow0$.  The damped direction then converges to the minimum-norm residual-matching direction.
SPRING-style momentum replaces the regularization center with a previous direction~\citep{goldshlager2024kaczmarz}.  The tangent space, energy metric, and retraction remain unchanged.

For a quadratic energy $E(u) = \frac12 a(u,u) - f(u)$, the deep Ritz method uses a symmetric and coercive bilinear form $a$ and $f\in X^*$, which gives
\begin{equation}
    G_E(\theta)_{ij} = a(\partial_{\theta_i} u_\theta, \partial_{\theta_j}u_\theta).
\end{equation}

\section{Theoretical Analysis}\label{sec:theory}

The following theorem is the central structural statement of EMNGD.
The theorem states that the EMNGD vector is the best feasible approximation to the function-space Newton correction under the energy metric.
\begin{theorem}[Main theorem: energy-manifold Newton projection]\label{thm:main_thm}
Assume that $D^2E(P(x))$ is symmetric, bounded, and coercive on $X$.
Let $H_x=D^2E(P(x))$ and $N_x\in X$ be the function-space Newton vector defined by
\begin{equation}\label{eq:newton_vector}
    H_x[N_x,v]=DE(P(x))[v]\qquad\text{for all }v\in X.
\end{equation}

Let $S_x=J_x(T_x\mathcal M)\subset X$.
If the undamped EMNGD equation has a solution, the push-forward satisfies
\begin{equation}\label{eq:pushforwardENGNewton}
    J_x\eta_x=\Pi_{S_x}^{H_x}N_x,
\end{equation}
where $\Pi_{S_x}^{H_x}$ is the $H_x$-orthogonal projection onto $S_x$.
In Euclidean coordinates, the projection identity becomes
\[
    DP_\theta\nabla^EL(\theta)
    =
    \Pi_{T_\theta \mathcal F_\Theta}^{D^2E(u_\theta)}
    \big(D^2E(u_\theta)^{-1}\nabla E(u_\theta)\big).
\]
\end{theorem}
\begin{proof}
The chain rule gives $dF_x[\zeta]=DE(P(x))[J_x\zeta]$.
The definition of $N_x$ gives $H_x[N_x,J_x\zeta]$.
The undamped EMNGD equation gives
\[
    H_x[J_x\eta_x,J_x\zeta]=H_x[N_x,J_x\zeta]
    \qquad\text{for all }\zeta\in T_x\mathcal M.
\]

Equivalently, $N_x-J_x\eta_x$ is $H_x$-orthogonal to every vector in $S_x$, while $J_x\eta_x\in S_x$.
The relation is precisely the Hilbert-space characterization of the orthogonal projection.
\end{proof}

\begin{corollary}[Quadratic energies]\label{cor:quadratic_energy_projection}
Let $E(u)=\frac12 a(u,u)-\ell(u)+c$, where $a:X\times X\to\mathbb R$ is symmetric, bounded, and coercive, and $\ell\in X^*$.
If $u^*$ is the unique minimizer, equivalently $a(u^*,v)=\ell(v)$ for all $v\in X$, then the undamped EMNGD direction satisfies
\begin{equation}\label{eq:pushforwardENG}
    J_x\eta_x=\Pi^a_{J_x(T_x\mathcal M)}(P(x)-u^*).
\end{equation}

The natural-gradient vector is the projected error $P(x)-u^*$, and the descent update $R_x(-\alpha\eta_x)$ moves toward the projected correction $u^*-P(x)$.
\end{corollary}
\begin{proof}
For a quadratic energy, $DE(u)[v]=a(u,v)-\ell(v)=a(u-u^*,v)$.
The Newton vector in the $a$-inner product is $u-u^*$.
The claim follows from Theorem~\ref{thm:main_thm} with $u=P(x)$ and $H_x=a$.
\end{proof}

\begin{proposition}[Damping as regularized projection]\label{prop:damped_regularized_projection}
Under the assumptions of Theorem~\ref{thm:main_thm}, let $N_x$ be the function-space Newton vector.
For $\lambda>0$, the damped EMNGD direction is the unique minimizer of
\begin{equation}\label{eq:regularized_projection_problem}
    \min_{\xi\in T_x\mathcal M}
    \left\{
       \frac12\|N_x-J_x\xi\|_{H_x}^2+\frac{\lambda}{2}\|\xi\|_0^2
    \right\}.
\end{equation}

Damping turns the exact tangent-space projection of the Newton vector into a Tikhonov-regularized tangent-space projection.
\end{proposition}
\begin{proof}
Differentiating the objective in~\eqref{eq:regularized_projection_problem} in the direction $\zeta$ gives the stationarity condition
\[
    H_x[J_x\xi,J_x\zeta]+\lambda g_x^0(\xi,\zeta)
    =H_x[N_x,J_x\zeta]=dF_x[\zeta],
\]
which is exactly~\eqref{eq:manifold_emngd_direction}.
Strict convexity follows from $\lambda>0$.
\end{proof}

Equations~\eqref{eq:pushforwardENGNewton} and~\eqref{eq:pushforwardENG} link parameter-space energy NG to a function-space Newton update.  For quadratic energies, the function-space natural-gradient vector is the current error $P(x)-u^*$.

\begin{proposition}[Descent direction]\label{prop:descent_direction}
Assume $\lambda>0$ and $\operatorname{grad}^0F(x)\neq0$.
Let $\eta_x$ be the EMNGD direction.
Then the retraction curve $\gamma(\alpha)=R_x(-\alpha\eta_x)$ satisfies
\begin{equation}\label{eq:descent_derivative}
    \left.\frac{d}{d\alpha}F(\gamma(\alpha))\right|_{\alpha=0}
    =-g_x^{E,\lambda}(\eta_x,\eta_x)<0.
\end{equation}
\end{proposition}
\begin{proof}
Since $R$ is a retraction, $\gamma'(0)=DR_x(0_x)[-\eta_x]=-\eta_x$.
The chain rule and~\eqref{eq:manifold_emngd_direction} give
\[
    \left.\frac{d}{d\alpha}F(\gamma(\alpha))\right|_{\alpha=0}
    =dF_x[-\eta_x]
    =-g_x^{E,\lambda}(\eta_x,\eta_x).
\]

Positive definiteness for $\lambda>0$ and $\operatorname{grad}^0F(x)\neq0$ imply $\eta_x\neq0$, so the derivative is strictly negative.
\end{proof}

\begin{assumption}[Uniform metric equivalence and retraction smoothness]\label{ass:convergence_emngd}
Let $\Omega=\{x\in\mathcal M:F(x)\leq F(x_0)\}$.
Assume that $F$ is bounded below on $\Omega$ and that there exist constants $0<m\leq M<\infty$ such that
\begin{equation}\label{eq:metric_equivalence}
    m\|\xi\|_0^2\leq g_x^{E,\lambda}(\xi,\xi)\leq M\|\xi\|_0^2
    \qquad (x\in\Omega,\ \xi\in T_x\mathcal M).
\end{equation}

Assume also that there exists $L_R>0$ such that every trial step considered by the line search satisfies
\begin{equation}\label{eq:retraction_smoothness}
    F(R_x(s))\leq F(x)+dF_x[s]+\frac{L_R}{2}\|s\|_0^2.
\end{equation}
\end{assumption}

\begin{theorem}[Global first-order convergence with Armijo line search]\label{thm:global_convergence}
Suppose Assumptions~\ref{ass:basic_emngd} and~\ref{ass:convergence_emngd} hold with $\lambda>0$.
At iteration $k$, compute the exact EMNGD direction $\eta_k=(A_{x_k}^\lambda)^{-1}\operatorname{grad}^0F(x_k)$.
Choose $\alpha_k$ by backtracking from $\alpha_0>0$ with contraction factor $\beta\in(0,1)$ until
\begin{equation}\label{eq:armijo}
    F(R_{x_k}(-\alpha_k\eta_k))
    \leq F(x_k)-c\alpha_k\,dF_{x_k}[\eta_k],
\end{equation}
holds for some $c\in(0,1)$.
Then the line search terminates, the iterates remain in $\Omega$, and
\begin{equation}\label{eq:grad_to_zero}
    \|\operatorname{grad}^0F(x_k)\|_0\to0.
\end{equation}

Every accumulation point is a first-order stationary point.
\end{theorem}
\begin{proof}
Let $g_k=\operatorname{grad}^0F(x_k)$ and $A_k=A_{x_k}^\lambda$.
The metric bounds imply
\[
    dF_{x_k}[\eta_k]=g_{x_k}^{E,\lambda}(\eta_k,\eta_k)
    \geq m\|\eta_k\|_0^2,
    \qquad
    dF_{x_k}[\eta_k]
    =g_{x_k}^0(g_k,A_k^{-1}g_k)
    \geq \frac1M\|g_k\|_0^2.
\]

Using~\eqref{eq:retraction_smoothness} with $s=-\alpha\eta_k$ gives
\[
    F(R_{x_k}(-\alpha\eta_k))
    \leq F(x_k)-\alpha\left(1-\frac{L_R\alpha}{2m}\right)dF_{x_k}[\eta_k].
\]

Every sufficiently small $\alpha$ satisfies~\eqref{eq:armijo}.  Backtracking terminates and returns a step bounded below by a positive constant.
The accepted steps yield $F(x_{k+1})\leq F(x_k)-C\|g_k\|_0^2$ for some $C>0$ independent of $k$.
Summing and using that $F$ is bounded below proves $\sum_k\|g_k\|_0^2<\infty$, so~\eqref{eq:grad_to_zero} holds.
Continuity of the Riemannian gradient gives stationarity of any accumulation point.
\end{proof}

\begin{proposition}[Inexact tangent solves]\label{prop:inexact_solve}
Let $\eta_x^*=(A_x^\lambda)^{-1}\operatorname{grad}^0F(x)$ be the exact direction.
Suppose an approximate direction $\widetilde\eta_x$ satisfies
\begin{equation}\label{eq:relative_energy_error}
    \|\widetilde\eta_x-\eta_x^*\|_{A_x^\lambda}
    \leq q\|\eta_x^*\|_{A_x^\lambda},
    \qquad q\in[0,1),
\end{equation}
where $\|\xi\|_{A_x^\lambda}^2=g_x^0(A_x^\lambda\xi,\xi)$.
Then
\begin{equation}\label{eq:inexact_descent}
    dF_x[\widetilde\eta_x]
    \geq (1-q)\|\eta_x^*\|_{A_x^\lambda}^2>0,
\end{equation}
whenever $\operatorname{grad}^0F(x)\neq0$.
As a result, $-\widetilde\eta_x$ remains a descent direction.
\end{proposition}
\begin{proof}
Write $e=\widetilde\eta_x-\eta_x^*$.
Since $A_x^\lambda\eta_x^*=\operatorname{grad}^0F(x)$,
\[
    dF_x[\widetilde\eta_x]
    =
    \|\eta_x^*\|_{A_x^\lambda}^2+
    \langle\eta_x^*,e\rangle_{A_x^\lambda}.
\]

Cauchy--Schwarz and~\eqref{eq:relative_energy_error} give the lower bound~\eqref{eq:inexact_descent}.
\end{proof}

\paragraph{Computational Complexity and Scalability.}\label{sec:complexity}

An EMNGD iteration comprises energy-operator construction or application, solution of a tangent linear system, and manifold operations.

We use $p$ for the ambient parameter dimension, $q=\dim\mathcal M$ for the intrinsic manifold dimension, $N$ for the number of residual samples, and $q=p$ for euclidean parameters.  
Let $\ell$ be the Nystr\"om rank and $m_{\mathrm{Krylov}}$ the number of Krylov iterations.  
The costs of one tangent-operator and one sample-space kernel application are denoted by $C_A$ and $C_K$, respectively. 
The following costs cover additional linear algebra after residual and derivative evaluation.  Residual and derivative costs depend on the PDE operator, network architecture, and automatic-differentiation implementation.

\begin{table}[t]
    \centering
    \scriptsize
    \caption{Dominant linear-algebra costs of EMNGD solvers.}
    \label{tab:complexity}
    \begin{tabular}{cccc}
        \toprule
        Method & Setup & Direction solve & Extra storage \\
        \midrule
        Direct tangent solve & $O(Nq^2)$ & $O(q^3)$ & $O(q^2)$ \\
        Exact Woodbury solve & $O(N^2q)$ & $O(N^3)$ & $O(N^2)$ \\
        Matrix-free tangent Krylov & --- & $O(m_{\mathrm{Krylov}}C_A)$ & $O(q)$ \\
        Nystr\"om-preconditioned Krylov & $O(\ell C_K+N\ell^2+\ell^3)$ & $O\!\left(m_{\mathrm{Krylov}}(C_K+N\ell+\ell^2)\right)$ & $O(N\ell+\ell^2)$ \\
        \bottomrule
    \end{tabular}
\end{table}

\begin{itemize}
    \item \textbf{Direct tangent solve.}  For residual and generalized Gauss--Newton models, explicit tangent coordinates give a residual Jacobian $J_x\in\mathbb R^{N\times q}$.  Forming $J_x^\top J_x$ costs $O(Nq^2)$, and a dense factorization costs $O(q^3)$.  The stated storage excludes $J_x$; retaining the Jacobian adds $O(Nq)$ memory.  Direct tangent solves are practical when $q$ is moderate.

    \item \textbf{Exact Woodbury solve.}  For a quadratic residual energy or a generalized Gauss--Newton pullback, the Woodbury identity replaces the tangent solve with a sample-space solve involving $K_x=J_xJ_x^\top$.  Explicit construction of $K_x$ costs $O(N^2q)$, dense solution costs $O(N^3)$, and the back-projection costs $O(Nq)$.  The route requires $O(N^2)$ additional storage and is favorable when $N\ll q$.

    \item \textbf{Matrix-free and Nystr\"om solvers.}  Matrix-free Krylov methods apply the damped tangent operator without forming a Gram matrix.  The solve costs $O(m_{\mathrm{Krylov}}C_A)$ and requires $O(q)$ working storage, apart from automatic-differentiation buffers.  Nystr\"om preconditioning constructs a rank-$\ell$ approximation to the sample-space kernel.  Each preconditioned Krylov iteration applies the exact kernel, so iterative convergence recovers the exact Woodbury direction.  A Nystr\"om sketch-and-solve method instead returns an approximate direction.

    \item \textbf{Geometric overhead and operating regimes.}  Let $C_\Pi$ and $C_R$ denote the costs of tangent projection and retraction.  An accepted update adds $C_\Pi+C_R$ to the linear-algebra cost.  Armijo backtracking with $n_{\mathrm{ls}}$ trial steps adds $O\!\left(n_{\mathrm{ls}}(C_F+C_R)\right)$, where $C_F$ is the energy-evaluation cost.  Direct tangent solves suit moderate $q$, exact Woodbury solves suit $N\ll q$, and Nystr\"om preconditioning reduces Krylov iterations for kernels with useful low-rank structure.
\end{itemize}

\section{Experiments}\label{sec:experiments}

The experiments address three questions.  First, does EMNGD recover the expected energy-metric behavior in the Euclidean specialization?  Second, do Woodbury duality and Nystr\"om preconditioning compute reliable tangent directions?  Third, how do the resulting solvers behave across PDEs, residual counts, and network sizes?  The benchmark tables use the Euclidean control $\mathcal M=\mathbb R^p$ with the additive retraction.  Separate residual-formulation diagnostics assess the Jacobian, Gramian, and sample-space solves.  The diagnostics test implementation consistency rather than architecture-matched manifold comparisons.

\subsection{Experimental Protocol}

\paragraph{Manifold parametrization.}
For a layer with weight matrix $W_\ell\in\mathbb R^{n_\ell\times n_{\ell-1}}$ we use the direction--scale decomposition
\begin{equation}\label{eq:exp:oblique_weight}
    W_\ell=\operatorname{Diag}\!\big(\exp(\rho_\ell)\big)\,Q_\ell^\top,
    \qquad
    Q_\ell\in\operatorname{Ob}(n_{\ell-1},n_\ell),
\end{equation}
where $\operatorname{Ob}(m,n)=\{Q\in\mathbb R^{m\times n}:\operatorname{diag}(Q^\top Q)=\mathbf 1\}$ is the oblique manifold of unit-norm columns.  The log-scales $\rho_\ell\in\mathbb R^{n_\ell}$ and biases $b_\ell\in\mathbb R^{n_\ell}$ remain Euclidean.  Every nonzero weight row is a length times a unit direction.  The direction--scale decomposition preserves the network function class and does not reduce the model.  Let $\Theta=((Q_\ell,\rho_\ell,b_\ell))_{\ell=1}^L$.  The parameter manifold is
\begin{equation}\label{eq:exp:product_manifold}
    \mathcal M=\prod_{\ell=1}^L\Big[\operatorname{Ob}(n_{\ell-1},n_\ell)\times\mathbb R^{n_\ell}\times\mathbb R^{n_\ell}\Big].
\end{equation}

The baseline metric uses the Frobenius metric on oblique factors and the Euclidean metric on scale and bias factors.  For an oblique factor, the tangent space, orthogonal projection, and normalization retraction are
\begin{align}\label{eq:exp:oblique_geometry}
    T_Q\operatorname{Ob}(m,n)&=\{\Xi:\operatorname{diag}(Q^\top\Xi)=0\},
    \qquad
    \Pi_Q(Z)=Z-Q\operatorname{Diag}\!\big(\operatorname{diag}(Q^\top Z)\big),
    \\
    R_Q(\Xi)&=(Q+\Xi)\operatorname{Diag}\!\Big(\operatorname{diag}\big((Q+\Xi)^\top(Q+\Xi)\big)\Big)^{-1/2},
    \notag
\end{align}
where the retraction $R_Q$ renormalizes the columns of $Q+\Xi$.  The Euclidean factors use $R_\rho(\delta\rho)=\rho+\delta\rho$ and $R_b(\delta b)=b+\delta b$.  The unconstrained Euclidean control drops the oblique constraint, so $\mathcal M=\mathbb R^p$ and $R_\Theta(v)=\Theta+v$.

The PDE benchmarks compare optimizers in the Euclidean control.  The product-manifold parametrization defines the constrained EMNGD setting.  The residual diagnostics test the tangent-space and sample-space computations separately.

\paragraph{Intrinsic EMNGD direction.}
Let $r(\Theta)$ collect the weighted interior, boundary, and initial residuals.  The tangent differential $\mathcal J_\Theta:T_\Theta\mathcal M\to\mathbb R^N$ maps a tangent direction to the residual change.  At iteration $k$, the damped EMNGD direction solves
\begin{equation}\label{eq:exp:tangent_system}
    \big(\mathcal J_{\Theta_k}^*\mathcal J_{\Theta_k}+\lambda_k I\big)\eta_k=\mathcal J_{\Theta_k}^*\,r(\Theta_k)
    \qquad\text{in }T_{\Theta_k}\mathcal M,
\end{equation}
The Woodbury form~\eqref{eq:emngd:mopt:manifold-woodbury} gives the same tangent direction with kernel $K_{\Theta_k}=\mathbf J_k\Pi_{\Theta_k}\mathbf J_k^\top$.  The line search evaluates $\Theta_k(\alpha)=R_{\Theta_k}(-\alpha\eta_k)$.
\eee

We globalize EMNGD with the retraction-based Armijo line search in Algorithm~\ref{alg:EMNGD}.  The search starts from the full trial step $\alpha=1$.  The projected-Newton interpretation of the undamped direction motivates that initial value.  Nonlinear realization maps, parameter manifolds, and retractions require a sufficient-decrease test before accepting the full trial step.

The experiments evaluate candidate step sizes on the geometric grid
\[
    \mathcal A=\{1,\beta,\beta^2,\ldots,\beta^{m_{\rm ls}}\}
    \subset(0,1],
    \qquad \beta\in(0,1).
\]

Energy evaluations on the grid can run in parallel.  The implementation selects the largest candidate satisfying the Armijo condition.  Further backtracking extends the grid when no candidate is accepted.
\begin{algorithm}[t]
\caption{Damped EMNGD with retraction-based Armijo line search}
\label{alg:EMNGD}
\begin{algorithmic}[1]
\STATE {\bfseries Input:} initial point $x_0\in\mathcal M$; positive damping parameters $\{\lambda_k\}_{k\geq0}$; Armijo parameter $c\in(0,1)$; backtracking factor $\beta\in(0,1)$; gradient tolerance $\varepsilon_{\rm grad}>0$; maximum iterations $N_{\max}$
\FOR{$k=0,\ldots,N_{\max}-1$}
    \STATE Compute $g_k=\operatorname{grad}^{0}F(x_k)\in T_{x_k}\mathcal M$.
    \IF{$\|g_k\|_{g^0_{x_k}}\leq\varepsilon_{\rm grad}$}
        \STATE {\bfseries stop}
    \ENDIF
    \STATE Define the $g^0_{x_k}$-self-adjoint tangent operator $A_k^{\lambda_k}:T_{x_k}\mathcal M\to T_{x_k}\mathcal M$ by
    \STATE \hspace{1em}$g^0_{x_k}(A_k^{\lambda_k}\xi,\zeta)=g^E_{x_k}(\xi,\zeta)+\lambda_k g^0_{x_k}(\xi,\zeta)$ for all $\xi,\zeta\in T_{x_k}\mathcal M$.
    \STATE Solve $A_k^{\lambda_k}\eta_k=g_k$ in $T_{x_k}\mathcal M$ exactly or to a prescribed inner-solver tolerance.
    \IF{$dF_{x_k}[\eta_k]\leq0$}
        \STATE Increase $\lambda_k$, or tighten the inner-solver tolerance, and recompute $\eta_k$.
    \ENDIF
    \STATE Set $\alpha_k\gets1$.
    \WHILE{$F\!\left(R_{x_k}(-\alpha_k\eta_k)\right)>F(x_k)-c\alpha_kdF_{x_k}[\eta_k]$}
        \STATE $\alpha_k\gets\beta\alpha_k$.
    \ENDWHILE
    \STATE Update $x_{k+1}=R_{x_k}(-\alpha_k\eta_k)$.
\ENDFOR
\end{algorithmic}
\end{algorithm}
For quadratic residual energies or generalized Gauss--Newton pullback metrics, Algorithm~\ref{alg:EMNGDWoodbury} uses an embedded product manifold with the metric induced by the ambient Euclidean product space.  The induced metric gives $\mathcal J_\Theta^*=\Pi_\Theta\mathbf J_\Theta^\top$.  General Riemannian metrics require metric-dependent projectors and adjoints.
\begin{algorithm}[t]
\caption{EMNGD with exact Woodbury and Nystr\"om-preconditioned solves}
\label{alg:EMNGDWoodbury}
\begin{algorithmic}[1]
\STATE {\bfseries Input:} initial point $\Theta_0\in\mathcal M$; positive damping parameters $\{\lambda_k\}_{k\geq0}$; solver mode $\mathsf{mode}\in\{\mathsf{Direct},\mathsf{NysPCG}\}$; Nystr\"om rank $\ell$; linear-solver tolerance $\varepsilon_{\rm lin}$; Armijo parameters $c,\beta\in(0,1)$; maximum iterations $N_{\max}$
\FOR{$k=0,\ldots,N_{\max}-1$}
    \STATE Evaluate $r_k=r(\Theta_k)\in\mathbb R^N$.
    \STATE Define the ambient residual Jacobian $\mathbf J_k=Dr(\Theta_k):\mathbb R^p\to\mathbb R^N$ through explicit assembly or matrix-free Jacobian products.
    \STATE Form or apply the $g^0$-orthogonal tangent projector $\Pi_k=\Pi_{\Theta_k}:\mathbb R^p\to T_{\Theta_k}\mathcal M$.
    \STATE Define $\mathcal J_k=\left.\mathbf J_k\right|_{T_{\Theta_k}\mathcal M}$ and $\mathcal J_k^*=\Pi_k\mathbf J_k^\top$.
    \STATE Define $K_k=\mathcal J_k\mathcal J_k^*=\mathbf J_k\Pi_k\mathbf J_k^\top$.
    \IF{$\mathsf{mode}=\mathsf{Direct}$}
        \STATE Solve $(K_k+\lambda_k I_N)a_k=r_k$ by a direct symmetric positive-definite solver.
    \ELSE
        \STATE Construct a rank-$\ell$ Nystr\"om approximation $\widehat K_k$ of $K_k$ and set $M_k=\widehat K_k+\lambda_k I_N$.
        \STATE Solve the \emph{exact} system $(K_k+\lambda_k I_N)a_k=r_k$ by preconditioned conjugate gradients with $M_k$ until $\|(K_k+\lambda_k I_N)a_k-r_k\|_2/\|r_k\|_2\leq\varepsilon_{\rm lin}$.
    \ENDIF
    \STATE Reconstruct $\eta_k=\mathcal J_k^*a_k=\Pi_k\mathbf J_k^\top a_k\in T_{\Theta_k}\mathcal M$.
    \IF{$dF_{\Theta_k}[\eta_k]\leq0$}
        \STATE Increase $\lambda_k$, or reduce $\varepsilon_{\rm lin}$, and recompute $a_k$ and $\eta_k$.
    \ENDIF
    \STATE Set $\alpha_k\gets1$ and apply the Armijo backtracking rule from Algorithm~\ref{alg:EMNGD}.
    \STATE Update $\Theta_{k+1}=R_{\Theta_k}(-\alpha_k\eta_k)$.
\ENDFOR
\end{algorithmic}
\end{algorithm}
Algorithm~\ref{alg:EMNGD} is the intrinsic, coordinate-free EMNGD method.  The algorithm defines the energy-metric equation on the current tangent space and uses retraction-based Armijo backtracking to obtain the next feasible iterate.

Algorithm~\ref{alg:EMNGDWoodbury} specializes EMNGD to quadratic residual energies or generalized Gauss--Newton pullback metrics on embedded product manifolds with the induced Euclidean product metric.  A direct sample-space solve computes the exact damped EMNGD direction through the Woodbury identity.  Nystr\"om-preconditioned Krylov iteration solves the same sample-space system to a prescribed tolerance.  Nystr\"om changes the conditioning of the inner solve but does not change the target EMNGD direction.

The unconstrained Euclidean specialization follows from $\mathcal M=\mathbb R^p$, $\Pi_\Theta=I_p$, and $R_\Theta(\eta)=\Theta+\eta$.  Section~\ref{sec:complexity} compares parameter-space, sample-space, and matrix-free costs.

In local coordinates at $\Theta_k$, let $G_{E,k}^\phi$, $G_{0,k}^\phi$, and $b_k$ denote the energy Gramian, baseline Gramian, and coordinate representation of $dF_{\Theta_k}$.  The damped EMNGD direction is the unique minimizer of the strictly convex tangent quadratic model
\begin{equation}\label{eq:coordinate_emngd_variational}
    v_k
    =
    \arg\min_{v\in\mathbb R^q}
    \left\{
        \frac12v^\top\big(G_{E,k}^\phi+\lambda_kG_{0,k}^\phi\big)v
        -b_k^\top v
    \right\},
    \qquad q=\dim\mathcal M.
\end{equation}

For quadratic residual energies or the generalized Gauss--Newton metric, the same direction solves the Tikhonov-regularized tangent least-squares problem
\begin{equation}\label{eq:tangent_tikhonov}
    v_k
    =
    \arg\min_{v\in\mathbb R^q}
    \left\{
        \frac12\|J_{\phi,k}v-r_k\|_2^2
        +\frac{\lambda_k}{2}v^\top G_{0,k}^\phi v
    \right\},
\end{equation}
where $J_{\phi,k}$ is the residual Jacobian in the selected tangent basis.

When the undamped tangent system is singular, the Moore--Penrose direction requires an explicit minimum-$g^0$-norm solution convention.  The computational algorithms use $\lambda_k>0$ to ensure uniqueness and improve conditioning.

\paragraph{Initialization, sampling, and budgets.}
For loss and Gram-matrix integrals, we \mzz use fixed regular grids or resampled random points. \eee  We initialize weights and biases from a zero-mean Gaussian with standard deviation $0.1$.  \mzz On the product manifold, each $Q_\ell$ is the columnwise normalization of the corresponding Gaussian matrix.  The log-scales $\rho_\ell$ and biases remain Euclidean.  The Euclidean control uses the unnormalized Gaussian initialization. \eee  Each PDE subsection states the collocation rule and iteration budget.  Tables list runtime separately from iteration counts.  The studies are not wall-clock-matched comparisons.

\subsubsection{Evaluation Metrics and Baselines}

We report relative $L^2$ error and, where derivative evaluations are available, relative $H^1$ error.  Evaluation uses denser quadrature than optimization.  The Euclidean-control studies compare stochastic gradient descent (SGD), Adam, BFGS \citep{nocedal1999numerical}, and ENGD.  SGD uses a logarithmic line-search grid.  Adam starts at $10^{-3}$.  After $1.5\times10^4$ steps, the learning rate decreases by a factor of $10^{-1}$ every $10^4$ steps.  The schedule stops at $10^{-7}$ or the iteration budget.

The convergence and mechanism figures display NGD, Hessian-free, Woodbury, SPRING, and Nystr\"om variants when the corresponding curve is labelled.  The displayed trajectories use the stated configurations and do not replace the 10-initialization table protocol.  The task-specific iteration budgets state the computational allocations.

The preliminary one-dimensional studies establish the Euclidean reduction and the solver identity.  Figure~\ref{fig:emngd_convergence_compare} summarizes loss and final relative $L^2$ error across the displayed PDEs.  The labelled EMNGD curve reaches the lowest displayed final errors.  Figure~\ref{fig:emgd_1d_experiments} shows that Woodbury ENGD follows the parameter-space ENGD trajectory.  The sample-space solve therefore changes the linear algebra, not the direction.

Figure~\ref{fig:poisson_1d_solution_comparison} gives a spatial check on one-dimensional Poisson.  EMNGD overlaps the reference solution and keeps the pointwise error near $10^{-7}$ or lower over most of the domain.  ENGD and NGD also track the reference, but retain interior errors near $10^{-5}$.  The diagnostic ordering supports the trajectory results but does not replace matched-budget comparisons.

\begin{figure}[t]
    \centering
    \includegraphics[width=\linewidth]{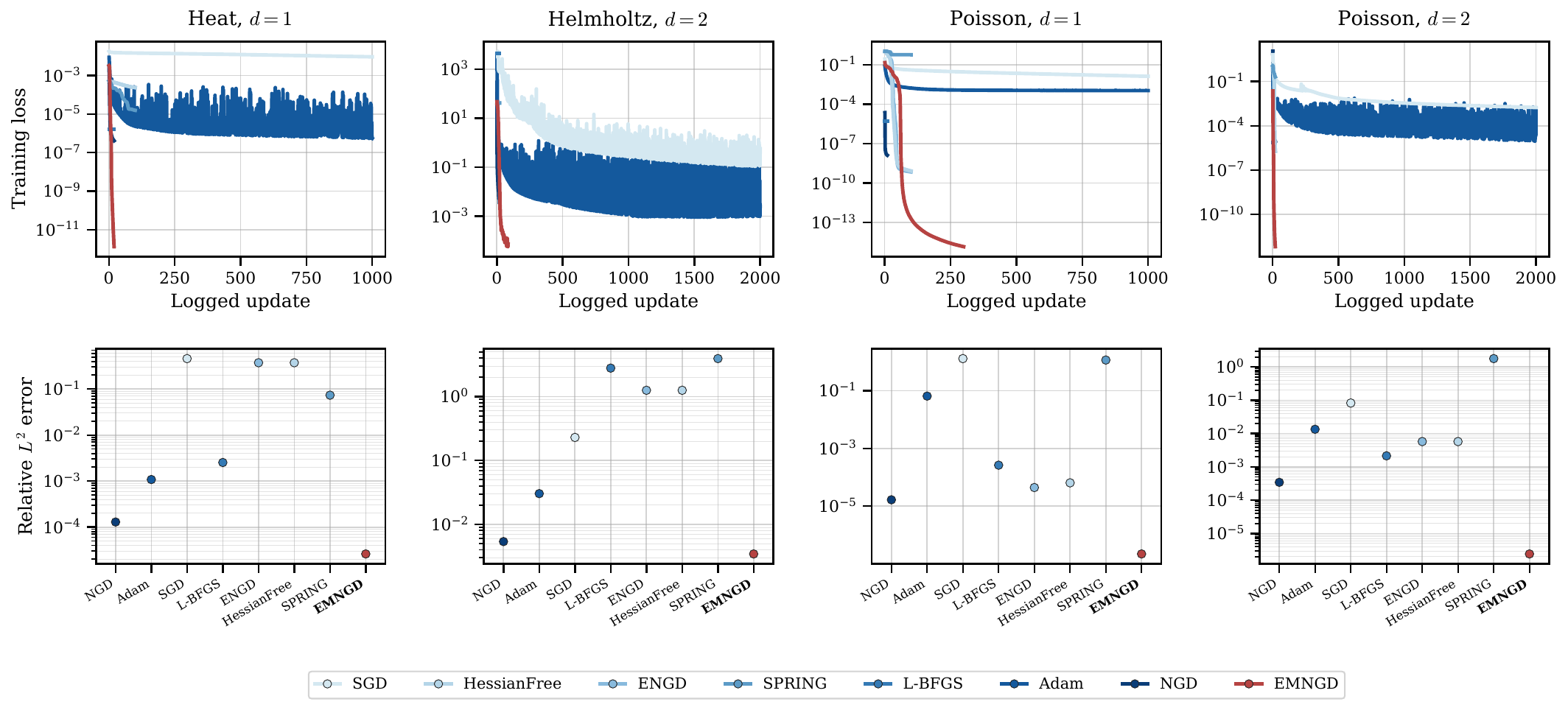}
    \vspace{-16pt}
    \caption{Training loss (top) and final relative $L^2$ error (bottom) across PDE benchmarks.}
    \label{fig:emngd_convergence_compare}
\end{figure}

\begin{figure}[t]
    \centering
    \includegraphics[width=\linewidth]{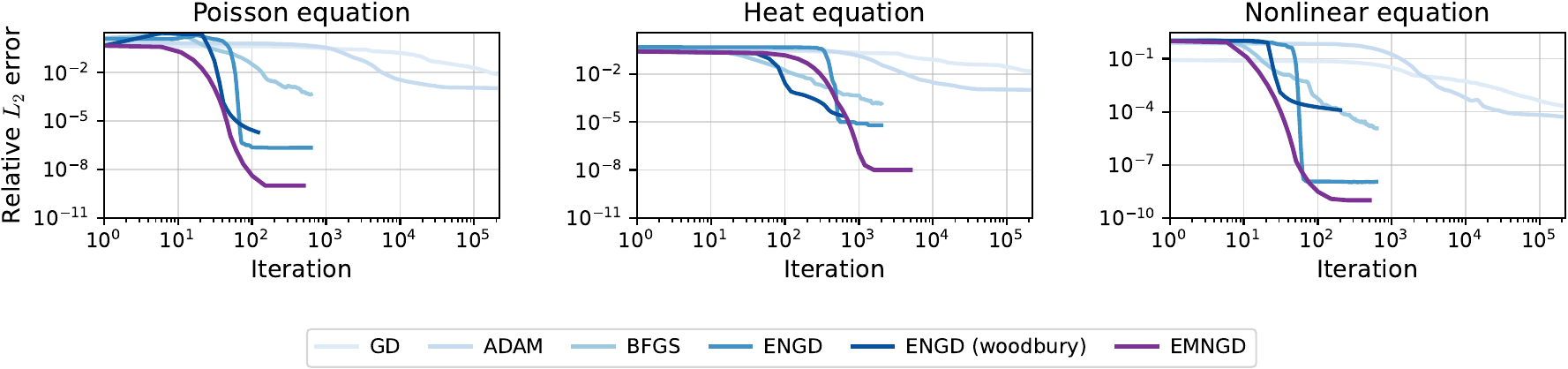}
    \vspace{-16pt}
    \caption{One-dimensional PDE benchmark.}
    \label{fig:emgd_1d_experiments}
\end{figure}

\begin{figure}[t]
    \centering
    \includegraphics[width=\linewidth]{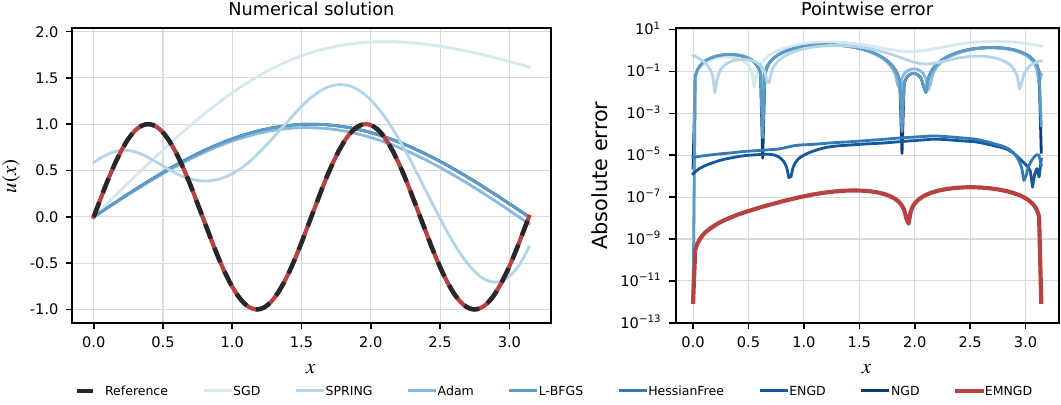}
    \caption{One-dimensional Poisson solutions and pointwise errors (log scale).}
    \label{fig:poisson_1d_solution_comparison}
\end{figure}

\subsubsection{Implementation Details}

We implement the solvers in JAX \citep{jax2018github} with automatic differentiation.  Least-squares solves use singular-value decomposition.  BFGS uses \texttt{jaxopt.BFGS}.  Unless stated otherwise, experiments run in double precision on one NVIDIA RTX 5090 Laptop GPU.  The implementation is available at \url{https://github.com/liangzhangyong/EMNGD}.

\subsection{Residual-Formulation Diagnostics}

The residual-formulation diagnostic verifies the hard-Dirichlet embedding and Woodbury solve before PDE accuracy comparisons.  The trial map $u_\theta(x)=\prod_{i=1}^2x_i(1-x_i)v_\theta(x)$ imposes the boundary condition by construction, leaving $u^\ast$ outside the training objective.  For 48 fixed interior residual points and a 337-parameter network, 80 updates reduce the residual loss by more than 14 orders of magnitude, from $5.521\times10^{1}$ to $2.941\times10^{-13}$, and yield a held-out relative $L^2$ error of $2.208\times10^{-4}$.  The diagnostic supports the correct interaction of the constraint embedding and Woodbury solver on the stated fixed-sample problem.

The primal and Woodbury directions agree to relative error $8.91\times10^{-9}$, verifying the dual implementation at numerical precision.  
Figures~\ref{fig:emngd-import-fisher} and~\ref{fig:emngd-import-spectrum} reveal evolving residual-Fisher geometry and a wide spectral range, which motivates damping in the sample-space solve. 
Figure~\ref{fig:emngd-woodbury-jacobian} displays the $48\times48$ kernel system that replaces the $337\times337$ parameter-space system, while Figure~\ref{fig:emngd-import-gramian} reconstructs the layerwise Gramian to error $4.36\times10^{-16}$.  
In contrast, Figure~\ref{fig:emngd-import-sharing} has relative Frobenius error $7.08\times10^{-1}$, so weight sharing changes the kernel and remains an approximation.

\begin{figure}[t]
    \centering
    \includegraphics[width=\linewidth]{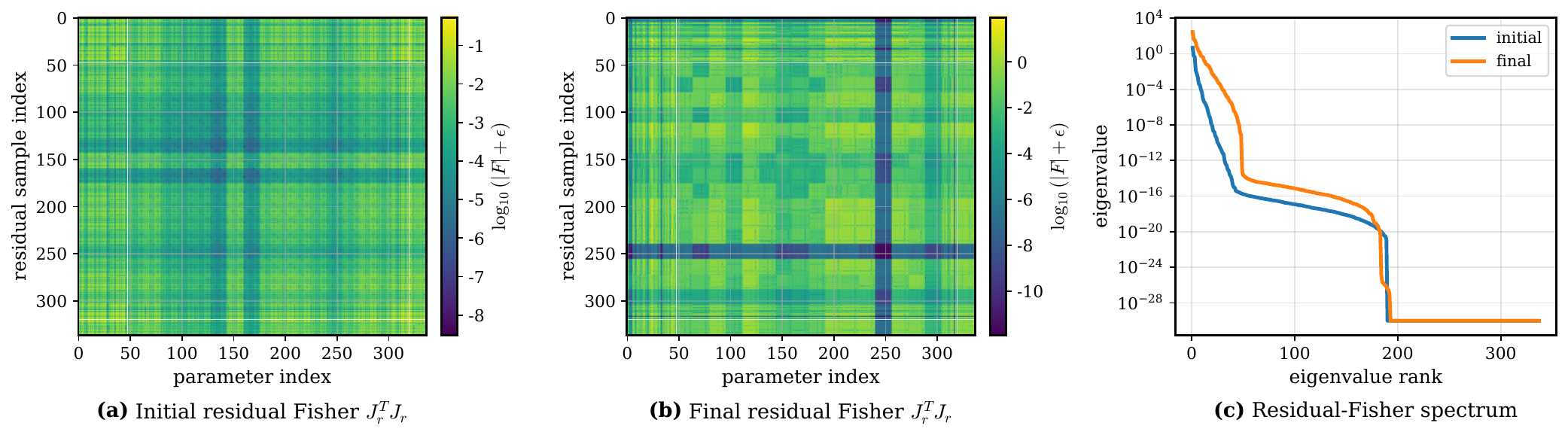}
    \vspace{-16pt}
    \caption{Residual-Fisher geometry for hard-Dirichlet EMNGD.}
    \label{fig:emngd-import-fisher}
\end{figure}

\begin{figure}[t]
    \centering
    \includegraphics[width=\linewidth]{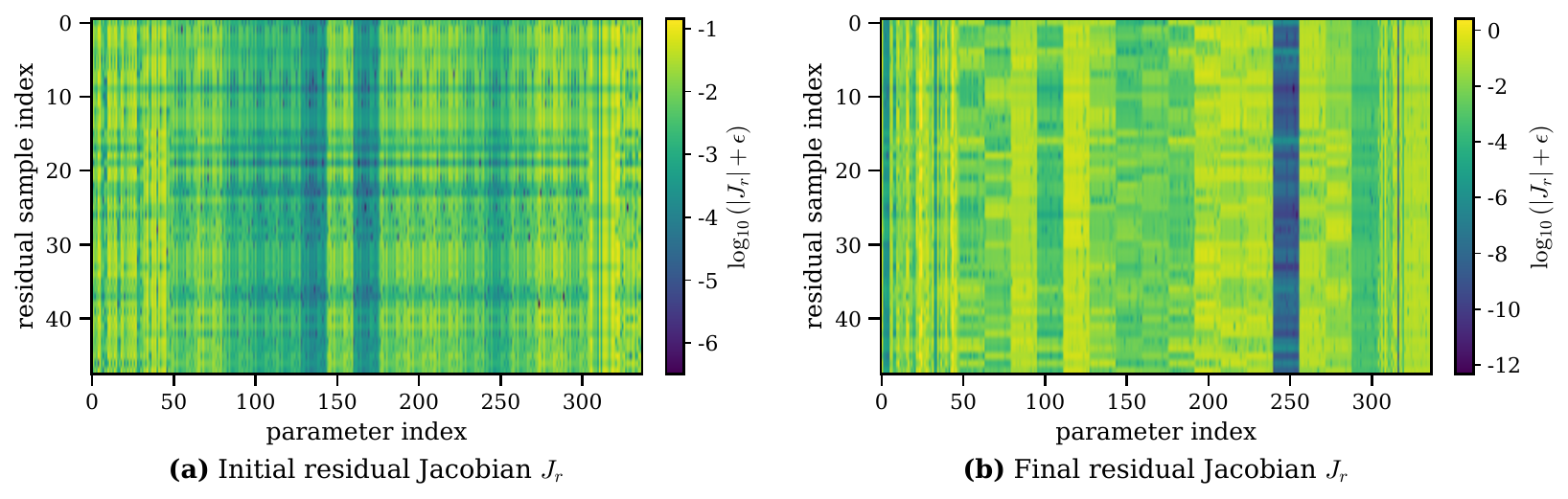}
    \caption{Residual Jacobians for Woodbury EMNGD on two-dimensional Poisson.}
    \label{fig:emngd-woodbury-jacobian}
\end{figure}

\begin{figure}[t]
    \centering
    \includegraphics[width=\linewidth]{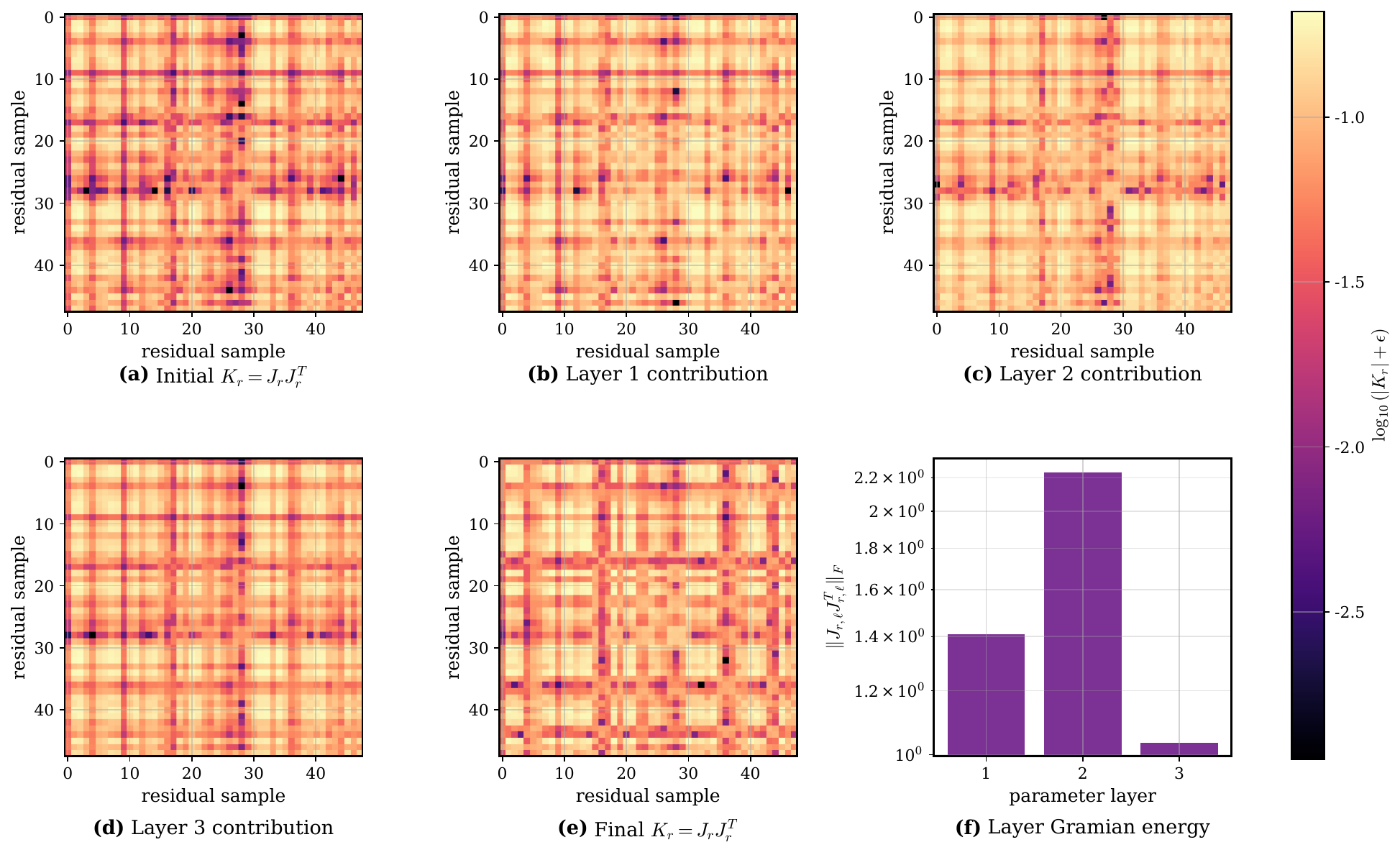}
    \caption{Layerwise contributions to the residual Gramian.}
    \label{fig:emngd-import-gramian}
\end{figure}

\begin{figure}[t]
    \centering
    \includegraphics[width=0.88\linewidth]{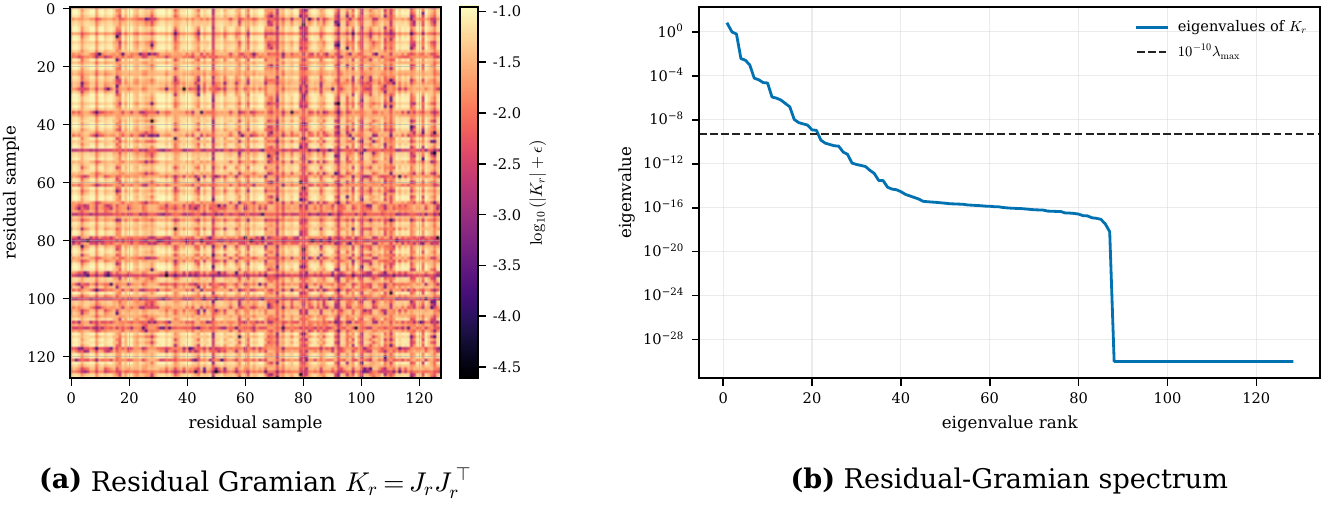}
    \caption{Sample-space residual-Gramian spectrum for $128$ residual samples.}
    \label{fig:emngd-import-spectrum}
\end{figure}

Figure~\ref{fig:emngd-local-preconditioner-effect} compares the exact Woodbury solve with rank-900 Nystr\"om preconditioning.  Both reduce residual loss to the $10^{-12}$ scale and reach relative $L^2$ errors near $10^{-8}$.  Woodbury finishes slightly lower ($1.09\times10^{-8}$ versus $1.45\times10^{-8}$).  The Nystr\"om trajectory closely follows Woodbury.

\begin{figure}[t]
    \centering
    \includegraphics[width=\linewidth]{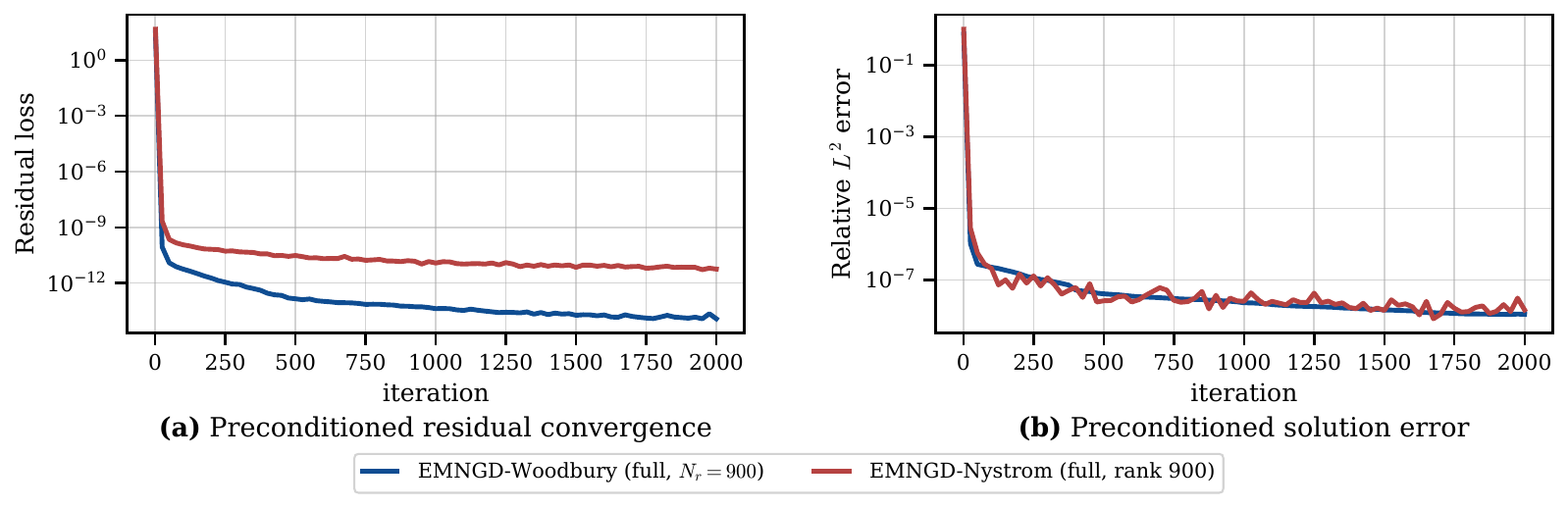}
    \caption{EMNGD with an exact Woodbury solve and rank-900 Nystr\"om preconditioning.  Left: residual loss.  Right: relative $L^2$ error.}
    \label{fig:emngd-local-preconditioner-effect}
\end{figure}

\begin{figure}[t]
    \centering
    \includegraphics[width=\linewidth]{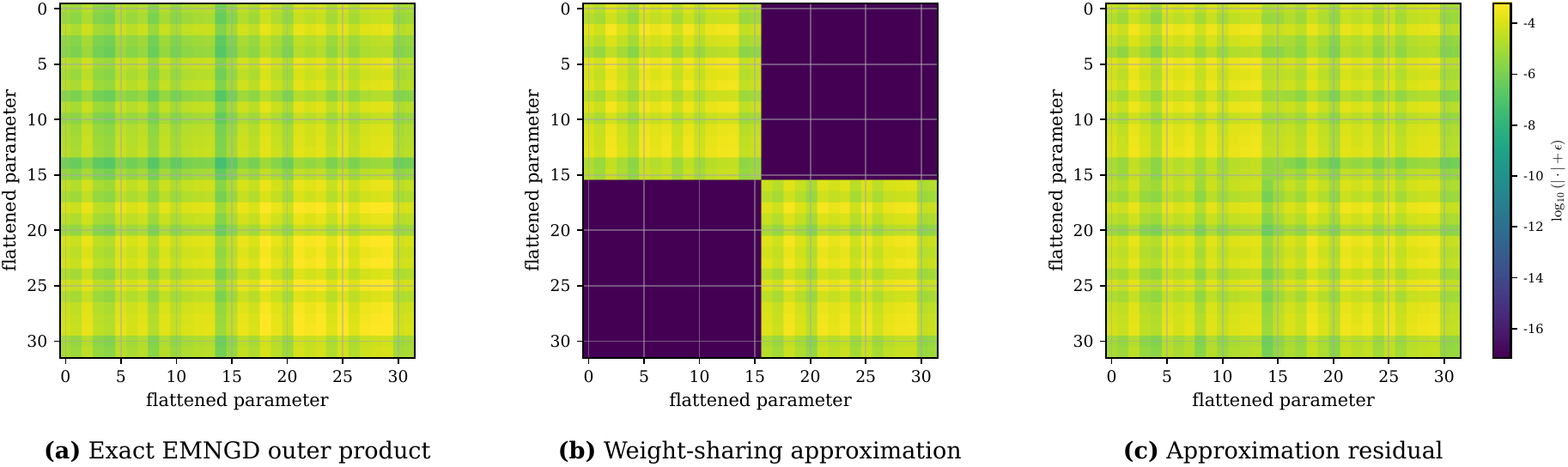}
    \caption{Weight-sharing approximation for an EMNGD residual-Jacobian block.}
    \label{fig:emngd-import-sharing}
\end{figure}

\subsection{Poisson Equation}\label{sec:Poisson_2d}
We consider the two-dimensional Poisson equation
\begin{equation*}
    -\Delta u (x,y) = f(x,y) = 2\pi^2\sin(\pi x) \sin(\pi y), 
\end{equation*}
on the unit square $[0,1]^2$ with zero boundary values. The solution is given by
\begin{equation*}
    u^*(x,y) = \sin(\pi x) \sin(\pi y),
\end{equation*}
and the PINNs loss of the problem is
\begin{align}\label{eq:poisson_loss}
\begin{split}
    L (\theta) & = \frac{1}{N_\Omega}\sum_{i=1}^{N_\Omega}(\Delta u_\theta(x_i,y_i) + f(x_i,y_i))^2 \\ & \qquad\qquad\quad  + \frac{1}{N_{\partial\Omega}}\sum_{i=1}^{N_{\partial\Omega}}u_\theta(x^b_i,y^b_i)^2,
\end{split}
\end{align}
where $\{(x_i,y_i)\}_{i=1,\dots,N_\Omega}$ denote the interior collocation points and $\{(x^b_i,y^b_i)\}_{i=1,\dots,N_{\partial\Omega}}$ denote the collocation points on $\partial\Omega$.  For the Poisson problem, the energy inner product on $H^2(\Omega)$ is
\begin{equation}\label{eq:poisson_energy_product}
    a(u,v) = \int_\Omega \Delta u \Delta v  \mathrm dx + \int_{\partial\Omega}u v \mathrm ds.
\end{equation}

The energy inner product is not coercive\footnote{The inner product is coercive with respect to the $H^{1/2}(\Omega)$ norm; see~\citep{muller2022notes}.} on $H^2(\Omega)$ and differs from the $H^2(\Omega)$ inner product. 
We approximate~\eqref{eq:poisson_energy_product} with the collocation points from~\eqref{eq:poisson_loss}.  The reproducibility rerun uses a common $2$--$32$--$1$ network.  SGD and Adam run for the recorded long-horizon updates.  BFGS, ENGD, and SPRING run for 50 updates.  EMNGD runs for 20 updates.

\begin{table}[t]
\centering
\scriptsize
\resizebox{0.75\linewidth}{!}{%
\begin{tabular}{ccccc}
\toprule
Method & Dim. & Steps & Loss & Relative $L^2$ error \\
\midrule
SGD & 2 & $199000$ & $2.192\times10^{-3}$ & $6.349\times10^{-3}$ \\
Adam & 2 & $200000$ & $1.200\times10^{-4}$ & $9.321\times10^{-4}$ \\
BFGS & 2 & $50$ & $3.366\times10^{-1}$ & $1.085\times10^{-1}$ \\
ENGD & 2 & $50$ & $1.046\times10^{2}$ & $4.639\times10^{-1}$ \\
SPRING & 2 & $50$ & $8.638\times10^{-6}$ & $5.158\times10^{-2}$ \\
\midrule
EMNGD & 2 & $20$ & $\mathbf{3.061\times10^{-11}}$ & $\mathbf{6.778\times10^{-9}}$ \\
\bottomrule
\end{tabular}%
}
\caption{Single-seed Poisson2D reproducibility reruns.}
\label{table:poisson}
\end{table}

\begin{table}[t]
\centering
\scriptsize
\resizebox{0.65\linewidth}{!}{%
\begin{tabular}{ccc}
\toprule
Method & Time per update & Full optimization time \\
\midrule
SGD & $1.8\times10^{-2}\,\mathrm{s}$ & $1\,\mathrm{h}$ \\
Adam & $3.7\times10^{-2}\,\mathrm{s}$ & $1\,\mathrm{h}\,6\,\mathrm{min}$ \\
BFGS & $1.8\,\mathrm{s}$ & $15\,\mathrm{min}$ \\
ENGD & $8.6\times10^{-2}\,\mathrm{s}$ & $43\,\mathrm{s}$ \\
KFAC & $4.899\times10^{-2}\,\mathrm{s}$ & $998.6\,\mathrm{s}$ \\
SPRING & $1.134\,\mathrm{s}$ & $56.68\,\mathrm{s}$ \\
\midrule
EMNGD & $\mathbf{1.271\times10^{-2}\,\mathrm{s}}$ & $\mathbf{25.42\,\mathrm{s}}$ \\
\bottomrule
\end{tabular}%
}
\caption{Poisson2D runtime records under the listed settings.}
\label{table:poisson_runtimes}
\end{table}

\begin{table}[t]
\centering
\scriptsize
\resizebox{0.75\linewidth}{!}{%
\begin{tabular}{ccccc}
\toprule
Method & Dim. & Steps & Loss & Relative $L^2$ error \\
\midrule
ANaGRAM & 2 & 50 & $4.536\times10^{1}$ & $1.181\times10^{0}$ \\
ENGD & 2 & 50 & $8.512\times10^{-9}$ & $4.493\times10^{-5}$ \\
HF-NGD & 2 & 50 & $1.033\times10^{-5}$ & $1.588\times10^{-3}$ \\
SNGD & 2 & 50 & $1.787\times10^{-7}$ & $2.217\times10^{-4}$ \\
EMNGD & 2 & 50 & $\mathbf{9.675\times10^{-12}}$ & $\mathbf{1.098\times10^{-6}}$ \\
\bottomrule
\end{tabular}%
}
\caption{Imported Poisson2D baseline integration results.}
\label{table:external_baseline_smoke}
\end{table}

\begin{table}[t]
\centering
\scriptsize
\resizebox{0.85\linewidth}{!}{%
\begin{tabular}{cccc}
\toprule
Method & Elapsed time (s) & Loss & Relative $L^2$ error \\
\midrule
SGD & $999.8$ & $7.248\times10^{-7}$ & $3.622\times10^{-5}$ \\
Adam & $999.7$ & $7.629\times10^{-7}$ & $9.399\times10^{-5}$ \\
Hessian-free & $1001.1$ & $1.484\times10^{-10}$ & $3.213\times10^{-7}$ \\
L-BFGS & $999.9$ & $1.024\times10^{-6}$ & $1.783\times10^{-4}$ \\
KFAC & $999.6$ & $4.282\times10^{-12}$ & $3.369\times10^{-7}$ \\
KFAC$^\ast$ & $999.7$ & $3.328\times10^{-12}$ & $1.990\times10^{-7}$ \\
ENGD (full) & $999.3$ & $3.553\times10^{-13}$ & $6.186\times10^{-8}$ \\
ENGD (layer-wise) & $999.9$ & $3.337\times10^{-13}$ & $6.262\times10^{-8}$ \\
ENGD (diagonal) & $999.8$ & $1.263\times10^{-4}$ & $5.371\times10^{-3}$ \\
\midrule
EMNGD--Nystr\"om & $149.9$ & $5.583\times10^{-12}$ & $1.451\times10^{-8}$ \\
EMNGD--Woodbury & $\mathbf{128.5}$ & $\mathbf{1.178\times10^{-14}}$ & $\mathbf{1.094\times10^{-8}}$ \\
\bottomrule
\end{tabular}%
}
\caption{Poisson2D endpoints of native solver implementations for $D=8{,}577$.}
\label{tab:poisson2d-native-coverage}
\end{table}

The two-dimensional study separates reproducibility, implementation coverage, and parameter scaling.  Table~\ref{table:poisson} uses a common $2$--$32$--$1$ network.  EMNGD reaches a relative $L^2$ error of $6.778\times10^{-9}$ after 20 updates.  The external runners in Table~\ref{table:external_baseline_smoke} use 257 parameters, whereas the native EMNGD run uses 8,577 parameters.  Table~\ref{tab:poisson2d-native-coverage} therefore records endpoint coverage rather than an architecture-matched ranking.

Figure~\ref{fig:poisson2d_l2_over_step} tests parameter scaling at $D=8{,}577$, $9{,}873$, and $116{,}097$.  The labelled EMNGD curve reaches relative $L^2$ errors near $10^{-7}$ within $10^3$ iterations in all three settings.  Table~\ref{tab:results-poisson} gives the corresponding terminal values.  Unequal stopping rules prevent a matched wall-clock or iteration-budget ranking.

\begin{figure}[t]
    \centering
    \includegraphics[width=\linewidth]{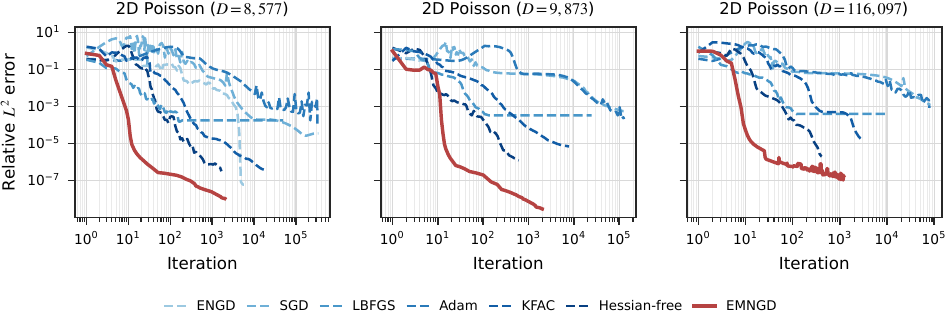}
    \caption{Two-dimensional Poisson relative $L^2$ error across three parameter dimensions.}
    \label{fig:poisson2d_l2_over_step}
\end{figure}

\begin{table*}[t]
\centering
\scriptsize
\resizebox{0.8\linewidth}{!}{%
\begin{tabular}{cccc}
\toprule
Architecture & Method & Loss & Relative $L^2$ error \\
\midrule
\multirow{7}{*}{$D=8{,}577$} & SGD & $7.248\times10^{-7}$ & $3.622\times10^{-5}$ \\
& Adam & $7.629\times10^{-7}$ & $9.399\times10^{-5}$ \\
& Hessian-free & $1.484\times10^{-10}$ & $3.213\times10^{-7}$ \\
& L-BFGS & $1.024\times10^{-6}$ & $1.783\times10^{-4}$ \\
& ENGD (full) & $3.553\times10^{-13}$ & $6.186\times10^{-8}$ \\
& KFAC & $4.282\times10^{-12}$ & $3.369\times10^{-7}$ \\
& EMNGD & $1.178\times10^{-14}$ & $\mathbf{1.094\times10^{-8}}$ \\
\midrule
\multirow{7}{*}{$D=9{,}873$} & SGD & $4.185\times10^{-6}$ & $5.575\times10^{-4}$ \\
& Adam & $1.667\times10^{-6}$ & $2.105\times10^{-4}$ \\
& Hessian-free & $7.291\times10^{-11}$ & $1.253\times10^{-6}$ \\
& L-BFGS & $2.412\times10^{-6}$ & $3.389\times10^{-4}$ \\
& ENGD (full) & $3.596\times10^{1}$ & $2.959\times10^{-1}$ \\
& KFAC & $2.108\times10^{-9}$ & $7.003\times10^{-6}$ \\
& EMNGD & $1.167\times10^{-15}$ & $\mathbf{2.973\times10^{-9}}$ \\
\midrule
\multirow{7}{*}{$D=116{,}097$} & SGD & $9.050\times10^{-6}$ & $1.219\times10^{-3}$ \\
& Adam & $2.991\times10^{-5}$ & $8.540\times10^{-4}$ \\
& Hessian-free & $1.762\times10^{-10}$ & $2.010\times10^{-6}$ \\
& L-BFGS & $2.405\times10^{-6}$ & $3.993\times10^{-4}$ \\
& ENGD (diagonal) & $2.637\times10^{-3}$ & $4.910\times10^{-2}$ \\
& KFAC & $5.829\times10^{-11}$ & $1.519\times10^{-5}$ \\
& EMNGD & $6.402\times10^{-10}$ & $\mathbf{1.767\times10^{-7}}$ \\
\bottomrule
\end{tabular}%
}
\caption{Terminal losses and relative $L^2$ errors for two-dimensional Poisson.}
\label{tab:results-poisson}
\end{table*}

Table~\ref{tab:results-poisson} lists the terminal values in Figure~\ref{fig:poisson2d_l2_over_step}.  The entries are archived trajectory endpoints and recorded hard residual-manifold EMNGD runs without $u^\ast$ in training.  The $D=116{,}097$ ENGD entry uses a diagonal approximation.  The endpoints do not form a matched-budget ranking.

\subsection{Five-Dimensional Poisson Equation}

We next consider the Poisson equation in five spatial dimensions:
    \begin{align*}
        -\Delta u & =  f \quad\quad \ \ \ \quad \quad \quad \text{in } [0, 1]^{5}, \\
        u(x) & = \sum_{k=1}^5 \sin(\pi x_k) \quad \ \text{on } \partial[0,1]^{5}.
    \end{align*}

We use the manufactured solution
\[ 
    u^\ast\colon \mathbb R^{5} \to\mathbb R, \quad   x \mapsto \sum_{k=1}^5 \sin(\pi x_k) 
\]
so $f=\pi^2u^\ast$.  We use the loss and energy inner product from Equations~\ref{eq:poisson_loss} and~\ref{eq:poisson_energy_product}.  Each optimization step draws $N_\Omega=3000$ interior points and $N_{\partial\Omega}=500$ boundary points.  The Euclidean-control network has five inputs, 64 hyperbolic-tangent hidden units, and one output.

The five-dimensional problem tests whether the sample-space solvers retain the energy-metric advantage as residual evaluation becomes more expensive.  Figure~\ref{fig:emngd_exp1_5d_poisson} reports relative $L^2$ error against iterations and wall-clock time.  Under the Euclidean-control protocol, energy-metric curves reach lower errors than the first-order baselines.  The comparison tests the Euclidean reduction, not a non-Euclidean manifold advantage.

\begin{figure}[t]
    \centering
    \includegraphics[width=0.92\linewidth]{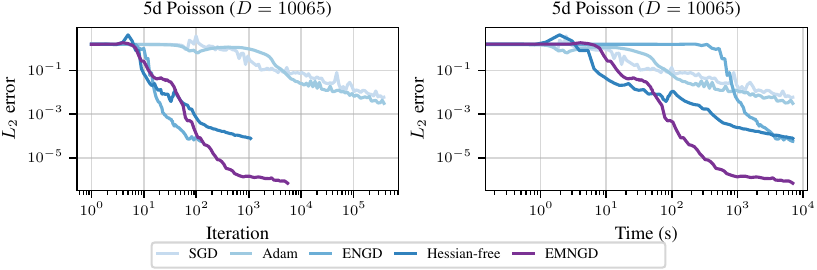}
    \caption{Five-dimensional Poisson: relative $L^2$ error versus iteration and time.}
    \label{fig:emngd_exp1_5d_poisson}
\end{figure}

Figure~\ref{fig:emngd_l2_loss_5d_poisson} separates training loss from evaluation error.  Energy-based curves reduce both quantities more rapidly than the first-order curves.  The displayed EMNGD trajectory reaches a low-error regime in both iteration and time.

\begin{figure}[t]
    \centering
    \includegraphics[width=0.92\linewidth]{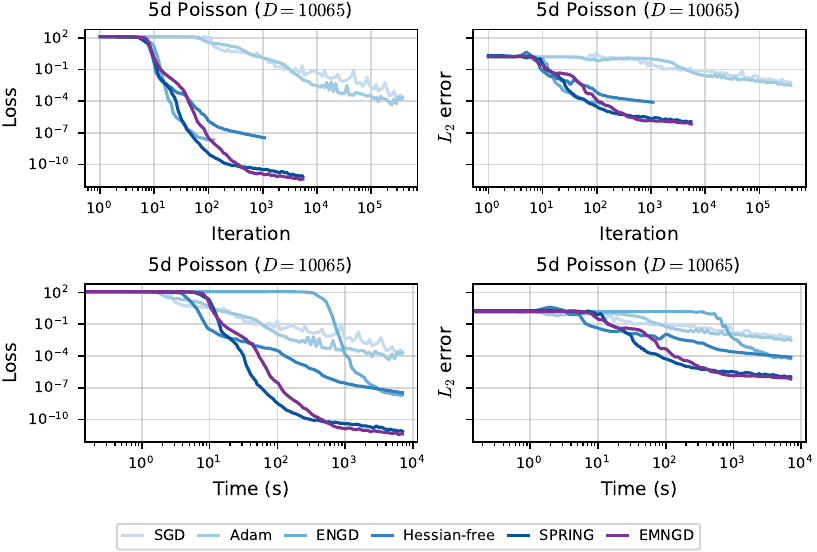}
    \caption{Loss and relative $L^2$ error for five-dimensional Poisson.}
    \label{fig:emngd_l2_loss_5d_poisson}
\end{figure}

Figure~\ref{fig:emngd_large_5d_poisson} then increases the number of collocation residuals from $N=1000$ to $N=10000$.  Woodbury and Nystr\"om variants reduce loss and relative $L^2$ error across all three sample sizes.  Nystr\"om follows the Woodbury convergence pattern while avoiding a dense sample-space solve.  The randomized SPRING curves vary more when the sampled kernel is less stable.

\begin{figure}[p]
    \centering
    \includegraphics[width=0.98\linewidth,height=0.78\textheight,keepaspectratio]{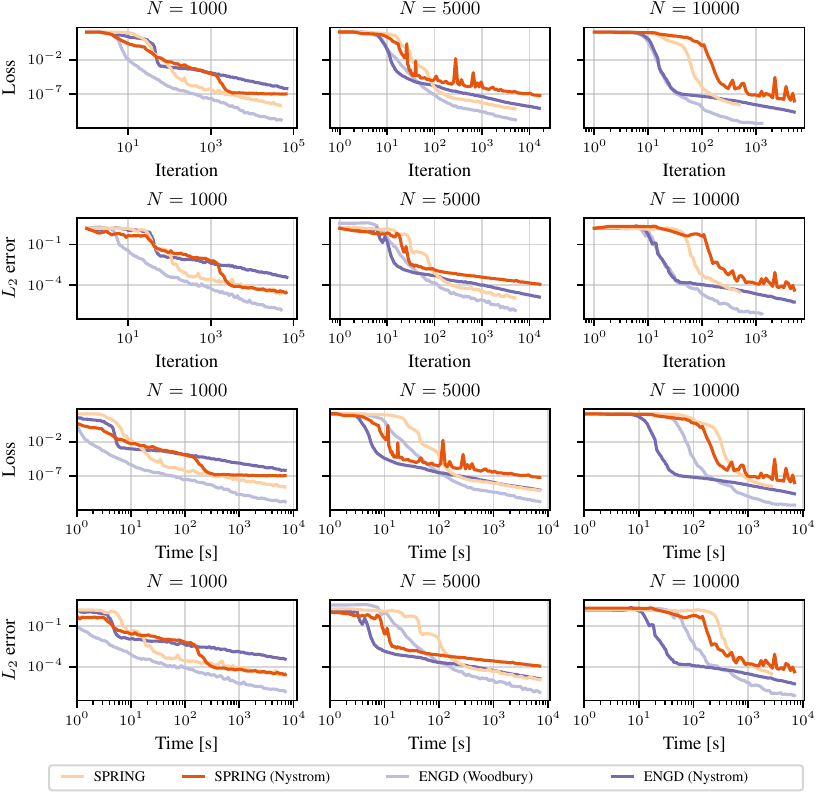}
    \caption{Large-sample five-dimensional Poisson benchmark.}
    \label{fig:emngd_large_5d_poisson}
\end{figure}

Figure~\ref{fig:emngd_metric_poisson} gives a metric-focused view of the same benchmark.  The energy-metric curves have lower loss and relative $L^2$ error than the first-order curves.  Hessian-free and SPRING improve on the first-order baselines, while the labelled EMNGD curve reaches the lowest displayed error range.

\begin{figure}[t]
    \centering
    \includegraphics[width=0.92\linewidth]{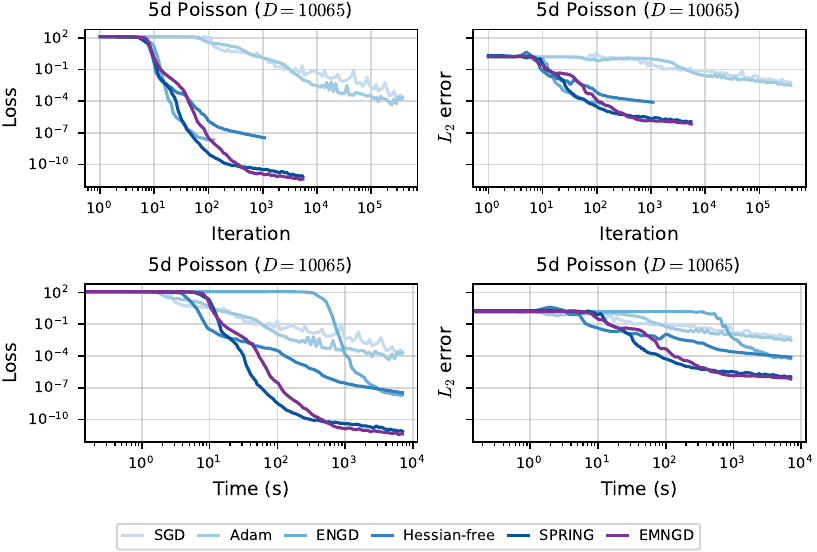}
    \caption{Metric comparison for five-dimensional Poisson.}
    \label{fig:emngd_metric_poisson}
\end{figure}

\eee

\subsection{Heat Equation}
Let us consider the one-dimensional heat equation
\begin{align*}
    \partial_t u(t,x) &= \frac{1}{4}\partial_x^2u(t,x) \quad \text{for }(t,x)\in[0,1]^2
    \\
    u(0,x) &= \sin(\pi x) \qquad\;\, \text{for }x\in[0,1]
    \\
    u(t,x) &= 0 \qquad\qquad\quad\text{for }(t,x)\in[0,1]\times\{0,1\}.
\end{align*}

The solution is given by
\begin{equation*}
    u^*(t,x) = \exp\left(-\frac{\pi^2 t}{4}\right)\sin(\pi x),
\end{equation*}
and the PINNs loss is
\begin{align*}
    L(\theta) &= \frac{1}{N_{\Omega_T}} \sum_{i=1}^{N_{\Omega_T}} \left( \partial_t u_\theta(t_i,x_i) - \frac14\partial_x^2 u_\theta(t_i, x_i) \right)^2 
    \\ 
    &\quad+ \frac{1}{N_\text{in}}\sum_{i=1}^{N_\Omega}\left(u_\theta(0,x_i^{\text{in}}) - \sin(\pi x_i^{\text{in}}) \right)^2
    \\&\quad +
    \frac{1}{N_{\partial\Omega}}\sum_{i=1}^{N_{\partial\Omega}}u_\theta(t^b_i,x^b_i)^2,
\end{align*}
where $\{ (t_i,x_i) \}_{i=1,\dots, N_{\Omega_T}}$ are interior space-time collocation points.  The set $\{ (t_i^b,x_i^b) \}_{i=1,\dots,N_{\partial\Omega}}$ contains spatial-boundary points.  The set $\{ (x_i^{\text{in}}) \}_{i=1,\dots,N_{\text{in}}}$ contains initial-condition points.  The energy inner product is defined on
\begin{equation*}
    a\colon\left(H^1(I,L^2(\Omega)) \cap L^2(I,H^2(\Omega))\right)^2 \to \mathbb R,
\end{equation*}
and given by
\begin{align*}
    a(u,v) &= \int_0^1\int_{\Omega}\left(\partial_t u - \frac14 \partial_x^2 u\right)\left(\partial_t v - \frac14 \partial_x^2 v\right)\,\mathrm dx\mathrm dt
    \\
    &\quad +
    \int_\Omega u(0,x)v(0,x)\, \mathrm dx + \int_{I\times\partial\Omega}uv \,\mathrm ds \mathrm dt.
\end{align*}

The heat problem tests convergence under a fixed recorded-time budget.  We discretize the energy inner product with the loss quadrature points.  The Euclidean-control experiment uses a width-64 hyperbolic-tangent network.  Table~\ref{table:heat_runtimes} reports updates, runtime, and final relative $L^2$ error.  EMNGD reaches the lowest recorded error in 199.03 seconds.

\begin{table}[t]
\centering
\scriptsize
\resizebox{0.9\linewidth}{!}{%
\begin{tabular}{ccccc}
\toprule
Method & Steps & Time per step & Run time & Relative $L^2$ error \\
\midrule
SGD & $288296$ & $3.465\times10^{-3}\,\mathrm{s}$ & $998.95\,\mathrm{s}$ & $2.825\times10^{-4}$ \\
Adam & $281802$ & $3.545\times10^{-3}\,\mathrm{s}$ & $998.96\,\mathrm{s}$ & $1.198\times10^{-4}$ \\
L-BFGS & $57770$ & $1.729\times10^{-2}\,\mathrm{s}$ & $998.86\,\mathrm{s}$ & $2.688\times10^{-4}$ \\
Hessian-free & $998$ & $1.002\,\mathrm{s}$ & $1000.46\,\mathrm{s}$ & $6.195\times10^{-7}$ \\
ENGD (full) & $4323$ & $2.310\times10^{-1}\,\mathrm{s}$ & $998.79\,\mathrm{s}$ & $3.047\times10^{-8}$ \\
KFAC & $13877$ & $7.196\times10^{-2}\,\mathrm{s}$ & $998.59\,\mathrm{s}$ & $1.883\times10^{-6}$ \\
SPRING & $1999$ & $1.607\times10^{-1}\,\mathrm{s}$ & $321.22\,\mathrm{s}$ & $9.927\times10^{-4}$ \\
\midrule
EMNGD & $1999$ & $9.957\times10^{-2}\,\mathrm{s}$ & $\mathbf{199.03\,\mathrm{s}}$ & $\mathbf{2.469\times10^{-9}}$ \\
\bottomrule
\end{tabular}%
}
\caption{Audited Heat-1D runtime and relative-error records across methods.}
\label{table:heat_runtimes}
\end{table}

Figure~\ref{fig:heat1d_l2_over_step} tests the same behavior at $D=4{,}417$, $5{,}441$, and $99{,}585$.  
The labelled EMNGD curve reaches relative $L^2$ errors near $10^{-8}$ within a few thousand iterations in all three settings.  The other displayed optimizers remain at higher errors over the plotted trajectories.
The runtimes are configuration-specific and do not represent hardware-independent complexity estimates.

\begin{figure}[t]
    \centering
    \includegraphics[width=\linewidth]{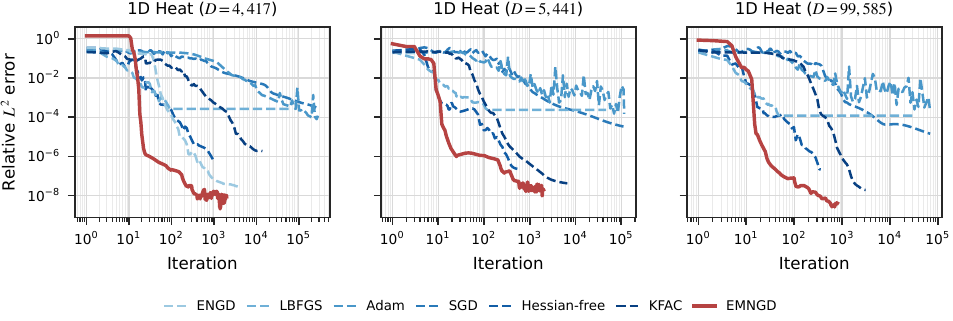}
    \caption{One-dimensional heat relative $L^2$ error across three parameter dimensions.}
    \label{fig:heat1d_l2_over_step}
\end{figure}

\section{Discussion}\label{sec:discussion}

\paragraph{Feasible and energy geometry.}
EMNGD combines two distinct structures.  
The parameter manifold defines admissible local variations, and retractions preserve feasibility of finite updates.  The pullback energy metric ranks the admissible variations by their function-space effects.  EMNGD therefore minimizes the energy-induced quadratic model directly on the tangent space.  The construction keeps the residual energy unchanged while incorporating parameter constraints through $T_x\mathcal M$ and $R_x$.  Post-hoc projection of an ambient ENGD direction generally solves a different local problem when projection and inversion do not commute.  The distinction matters whenever the manifold geometry restricts directions that have strong energy curvature in the ambient space.

\paragraph{Geometry and linear algebra.}
The geometric definition does not depend on a particular tangent solver.  For quadratic residual energies and generalized Gauss--Newton pullbacks, Woodbury gives an exact sample-space form of the damped tangent direction.  The identity replaces the primal system with the kernel $K_x=\mathbf J_x\Pi_x\mathbf J_x^\top$ and reconstructs the tangent vector through $\Pi_x\mathbf J_x^\top$.  Sample-space natural-gradient solves also arise in MinSR methods for variational Monte Carlo~\citep{chen2023efficient,rende2024simple}.  For EMNGD, Woodbury changes the linear algebra but not the tangent metric, feasible space, or retraction.

Nystr\"om methods have two different roles.  Sketch-and-solve uses a low-rank kernel approximation and produces an approximate tangent direction.  Nystr\"om-preconditioned Krylov iteration uses the approximation only to accelerate the exact Woodbury system.  Iterative convergence then recovers the same damped direction as the direct sample-space solve.  The distinction is important when interpreting accuracy and computational cost.

Figure~\ref{fig:discussion_minsr_woodbury_emngd} isolates the solver behavior in the overparameterized regime $N\ll p$.  Panel~(a) compares the scaling of the primal and sample-space formulations as the ambient dimension grows.  The sample-space cost remains controlled by the residual count.  Panel~(b) records primal--dual direction agreement at numerical precision and shows the dependence of a Nystr\"om approximation on sketch rank.  Panel~(c) compares the corresponding linearized residual trajectories.  The diagnostic separates exact Woodbury duality from acceleration strategies that approximate the kernel.

\begin{figure}[t]
    \centering
    \includegraphics[width=\linewidth]{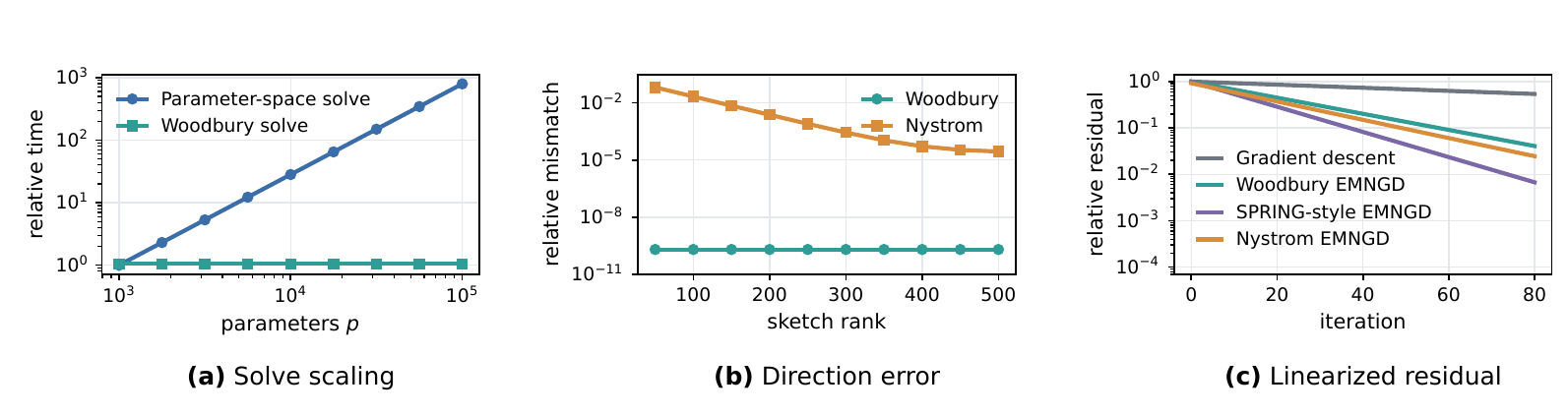}
    \vspace{-16pt}
    \caption{Woodbury and Nystr\"om diagnostics for tangent-solver scaling and agreement.}
    \label{fig:discussion_minsr_woodbury_emngd}
\end{figure}

\paragraph{Operating regimes and limitations.}
Woodbury is most useful when the residual count is below the intrinsic parameter dimension.  The computational bottleneck then moves from a parameter-space system to an $N\times N$ sample kernel.  Large residual sets still require quadratic kernel storage and can require expensive dense solves.  The formulation therefore does not remove the cost of residual evaluation, Jacobian products, or sample-space conditioning.  Matrix-free products and iterative solves become necessary when dense kernels no longer fit in memory.

Figure~\ref{fig:discussion_minsr_limitations} summarizes three practical limits.  Sample-kernel memory grows quadratically with $N$.  Reducing the damping parameter approaches the undamped natural-gradient direction but worsens conditioning and amplifies residual perturbations.  Residual subsampling replaces the full kernel by a sampled system, so the computed direction can differ from the full-residual direction.  Nystr\"om preconditioning is effective when the regularized kernel has a low effective dimension.  Excessive rank reduction can instead degrade the direction.  Large damping improves conditioning but moves the update toward the baseline Riemannian gradient.

\begin{figure}[t]
    \centering
    \includegraphics[width=\linewidth]{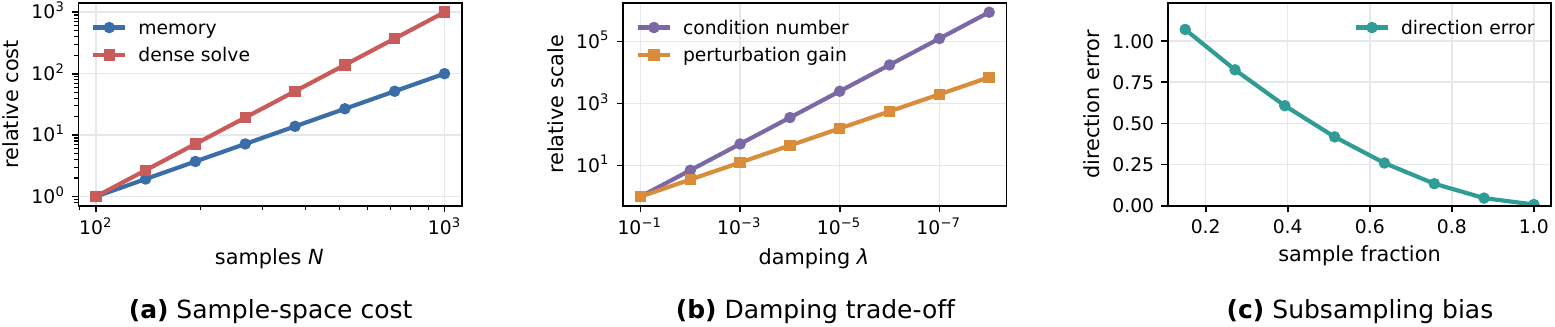}
    \caption{Sample-kernel growth, damping sensitivity, and residual-subsampling effects.}
    \label{fig:discussion_minsr_limitations}
\end{figure}

The residual-Jacobian diagnostic in Figure~\ref{fig:discussion_engd_woodbury_minsr_failure} provides the same perspective for the ENGD--Woodbury implementation.  The kernel spectrum spans several orders of magnitude, which explains the sensitivity of the undamped solve to small perturbations.  
The result also quantifies the change in direction caused by residual batches and insufficient Nystr\"om rank.  Such effects are numerical properties of the sampled tangent system rather than changes in the EMNGD geometry.  Damping, validation batches, and controlled Krylov tolerances provide practical safeguards, but no setting removes the underlying sample-space trade-off.

\begin{figure}[t]
    \centering
    \includegraphics[width=\linewidth]{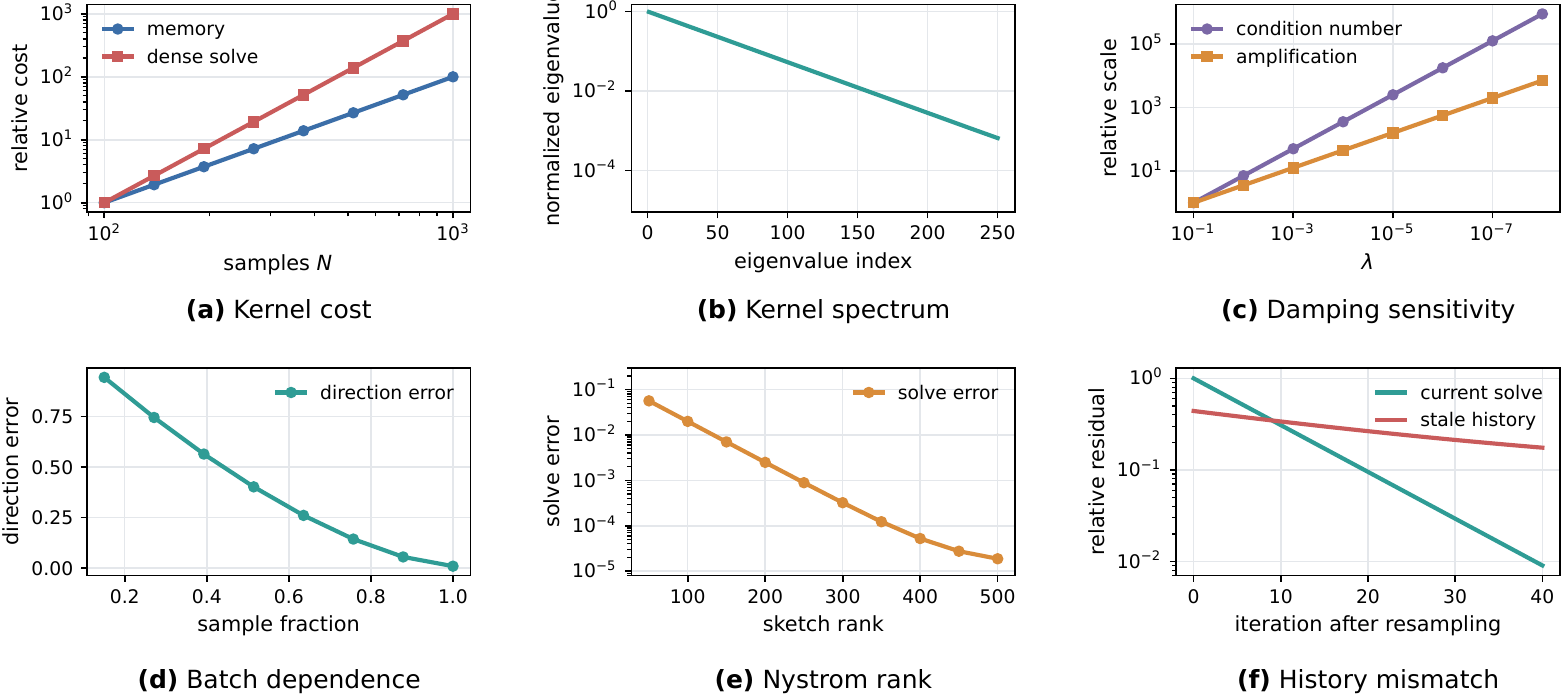}
    \vspace{-18pt}
    \caption{Residual-Jacobian sensitivity to damping, sampling, and Nystr\"om rank.}
    \label{fig:discussion_engd_woodbury_minsr_failure}
\end{figure}

\paragraph{Evidence scope.}
The reported benchmarks verify Euclidean reduction, primal--dual Woodbury agreement, and high-accuracy optimization on the stated neural PDE problems.  The residual-formulation tests also confirm that hard boundary embeddings can be handled without using the exact PDE solution during training.  Several tables use different architectures, stopping rules, or hardware settings.  Such records provide solver coverage and endpoint evidence rather than a uniform ranking.  Most accuracy studies evaluate the Euclidean specialization.  Architecture-matched tests on genuinely constrained neural PDE models will clarify the empirical value of the manifold component beyond the geometric guarantees.

\section{Conclusion}
We introduced \EMNGDfull{}, an intrinsic extension of ENGD from an unconstrained Euclidean parameter domain to a Riemannian parameter manifold.  
EMNGD combines the feasible geometry of the parameter manifold with the energy geometry induced by the neural PDE objective.  
The method solves the energy-metric quadratic model over feasible tangent directions and uses a retraction to preserve the parameter constraints.

Under coercivity, the push-forward of the undamped EMNGD direction is the best feasible approximation to the function-space Newton vector in the energy metric.  
The analysis also establishes well-posedness under damping, coordinate invariance, exact reduction to ENGD in Euclidean space, and global first-order convergence with Armijo backtracking.  
Controlled inexact tangent solves retain the descent property.
For quadratic residual energies and generalized Gauss--Newton pullbacks, the Woodbury identity gives an exact sample-space representation of the damped tangent direction.  
Nystr\"om sketching produces an approximate direction, whereas Nystr\"om-preconditioned Krylov solves recover the exact direction after convergence.  
The resulting solvers are most useful when the residual count is smaller than the intrinsic parameter dimension.

Experiments verify Euclidean reduction, primal--dual agreement, and accurate neural PDE optimization under the reported settings.  
The framework provides a geometric basis for constrained neural PDE solvers without altering the residual energy.  
Future work can evaluate architecture-matched constrained manifolds and develop matrix-free solvers for larger residual systems.

\acks{The study was supported by the National Natural Science Foundation of China (12572138). 
The authors declare no competing interests.}

\bibliography{bib}

@book{nocedal1999numerical,
  title={Numerical optimization},
  author={Nocedal, Jorge and Wright, Stephen J},
  year={1999},
  publisher={Springer}
}

@misc{jax2018github,
  author={Bradbury, James and Frostig, Roy and Hawkins, Peter and Johnson, Matthew James and Leary, Chris and Maclaurin, Dougal and Necula, George and Paszke, Adam and VanderPlas, Jake and Wanderman-Milne, Skye and Zhang, Qiao},
  title={{JAX}: composable transformations of {Python}+{NumPy} programs},
  year={2018},
  url={http://github.com/google/jax}
}

@inproceedings{mullerposition,
  title={Position: Optimization in {SciML} Should Employ the Function Space Geometry},
  author={M{\"u}ller, Johannes and Zeinhofer, Marius},
  booktitle={Forty-first International Conference on Machine Learning},
  year={2024}
}

@inproceedings{muller2023achieving,
  title={Achieving high accuracy with {PINN}s via energy natural gradient descent},
  author={M{\"u}ller, Johannes and Zeinhofer, Marius},
  booktitle={International Conference on Machine Learning},
  pages={25471--25485},
  year={2023},
  organization={PMLR}
}

@article{jnini2024gauss,
  title={{Gauss-Newton} Natural Gradient Descent for Physics-Informed Computational Fluid Dynamics},
  author={Jnini, Anas and Vella, Flavio and Zeinhofer, Marius},
  journal={arXiv preprint arXiv:2402.10680},
  year={2024}
}

@article{dangel2024kronecker,
  title={Kronecker-Factored Approximate Curvature for Physics-Informed Neural Networks},
  author={Dangel, Felix and M{\"u}ller, Johannes and Zeinhofer, Marius},
  journal={arXiv preprint arXiv:2405.15603},
  year={2024}
}

@article{chen2023efficient,
  title={Efficient optimization of deep neural quantum states toward machine precision},
  author={Chen, Ao and Heyl, Markus},
  journal={arXiv preprint arXiv:2302.01941},
  year={2023}
}

@article{rende2024simple,
  title={A simple linear algebra identity to optimize large-scale neural network quantum states},
  author={Rende, Riccardo and Viteritti, Luciano Loris and Bardone, Lorenzo and Becca, Federico and Goldt, Sebastian},
  journal={Communications Physics},
  volume={7},
  number={1},
  pages={260},
  year={2024},
  publisher={Nature Publishing Group UK London}
}

@article{goldshlager2024kaczmarz,
  title={A {K}aczmarz-inspired approach to accelerate the optimization of neural network wavefunctions},
  author={Goldshlager, Gil and Abrahamsen, Nilin and Lin, Lin},
  journal={Journal of Computational Physics},
  volume={516},
  pages={113351},
  year={2024},
  publisher={Elsevier}
}

@article{frangella2023randomized,
  title={Randomized {N}ystr{\"o}m preconditioning},
  author={Frangella, Zachary and Tropp, Joel A and Udell, Madeleine},
  journal={SIAM Journal on Matrix Analysis and Applications},
  volume={44},
  number={2},
  pages={718--752},
  year={2023},
  publisher={SIAM}
}

@article{gittens2016nystrom,
  title={Revisiting the {N}ystr{\"o}m Method for Improved Large-Scale Machine Learning},
  author={Gittens, Alex and Mahoney, Michael W.},
  journal={Journal of Machine Learning Research},
  volume={17},
  number={117},
  pages={1--65},
  year={2016}
}

@article{nie2026nystrom,
  title={{N}ystr{\"o}m Approximation on Manifolds},
  author={Nie, Hantao and Gao, Bin and Han, Andi and Jawanpuria, Pratik and Mishra, Bamdev and Wen, Zaiwen},
  journal={arXiv preprint arXiv:2605.14933},
  year={2026}
}

@inproceedings{novak2022fast,
  title={Fast finite width neural tangent kernel},
  author={Novak, Roman and Sohl-Dickstein, Jascha and Schoenholz, Samuel S},
  booktitle={International Conference on Machine Learning},
  pages={17018--17044},
  year={2022},
  organization={PMLR}
}

@article{hao2021efficient,
  title={An efficient greedy training algorithm for neural networks and applications in {PDEs}},
  author={Hao, Wenrui and Jin, Xianlin and Siegel, Jonathan W and Xu, Jinchao},
  journal={arXiv preprint arXiv:2107.04466},
  year={2021}
}

@inproceedings{
zeng2022competitive,
title={Competitive Physics Informed Networks},
author={Qi Zeng and Spencer H Bryngelson and Florian Tobias Schaefer},
booktitle={ICLR 2022 Workshop on Gamification and Multiagent Solutions},
year={2022},
url={https://openreview.net/forum?id=rMz_scJ6lc}
}

@article{lu2021deepxde,
  title={DeepXDE: A deep learning library for solving differential equations},
  author={Lu, Lu and Meng, Xuhui and Mao, Zhiping and Karniadakis, George Em},
  journal={SIAM Review},
  volume={63},
  number={1},
  pages={208--228},
  year={2021},
  publisher={SIAM}
}

@article{wang2022respecting,
  title={Respecting causality is all you need for training physics-informed neural networks},
  author={Wang, Sifan and Sankaran, Shyam and Perdikaris, Paris},
  journal={arXiv preprint arXiv:2203.07404},
  year={2022}
}

@article{zapf2022investigating,
  title={Investigating molecular transport in the human brain from MRI with physics-informed neural networks},
  author={Zapf, Bastian and Haubner, Johannes and Kuchta, Miroslav and Ringstad, Geir and Eide, Per Kristian and Mardal, Kent-Andre},
  journal={Scientific Reports},
  volume={12},
  number={1},
  pages={1--12},
  year={2022},
  publisher={Nature Publishing Group}
}

@article{daw2022rethinking,
  title={Rethinking the importance of sampling in physics-informed neural networks},
  author={Daw, Arka and Bu, Jie and Wang, Sifan and Perdikaris, Paris and Karpatne, Anuj},
  journal={arXiv preprint arXiv:2207.02338},
  year={2022}
}

@article{wu2023comprehensive,
  title={A comprehensive study of non-adaptive and residual-based adaptive sampling for physics-informed neural networks},
  author={Wu, Chenxi and Zhu, Min and Tan, Qinyang and Kartha, Yadhu and Lu, Lu},
  journal={Computer Methods in Applied Mechanics and Engineering},
  volume={403},
  pages={115671},
  year={2023},
  publisher={Elsevier}
}

@inproceedings{muller2022error,
  title={Error estimates for the deep Ritz method with boundary penalty},
  author={M{\"u}ller, Johannes and Zeinhofer, Marius},
  booktitle={Mathematical and Scientific Machine Learning},
  pages={215--230},
  year={2022},
  organization={PMLR}
}

@incollection{schwedes2017mesh,
  title={Mesh dependence in {PDE}-constrained optimisation},
  author={Schwedes, Tobias and Ham, David A and Funke, Simon W and Piggott, Matthew D},
  booktitle={Mesh Dependence in {PDE}-Constrained Optimisation},
  pages={53--78},
  year={2017},
  publisher={Springer}
}

@article{schwedes2016iteration,
  title={An iteration count estimate for a mesh-dependent steepest descent method based on finite elements and {R}iesz inner product representation},
  author={Schwedes, Tobias and Funke, Simon W and Ham, David A},
  journal={arXiv preprint arXiv:1606.08069},
  year={2016}
}

@article{Mueller2022Convergence, 
title = {Geometry and Convergence of Natural Policy Gradients},
year = {2022}, 
author = {M\"uller, Johannes and Mont\'ufar, Guido},
journal = {MPI MiS Preprint 31/2022},
url ={https://www.mis.mpg.de/publications/preprints/2022/prepr2022-31.html}}

@article{Li2018natural,
	Author = {Li, Wuchen and Mont{\'u}far, Guido},
	Da = {2018/12/01},
	noDoi = {10.1007/s41884-018-0015-3},
	Id = {Li2018},
	Isbn = {2511-249X},
	Journal = {Information Geometry},
	Number = {2},
	Pages = {181--214},
	Title = {Natural gradient via optimal transport},
	Ty = {JOUR},
	Url = {https://doi.org/10.1007/s41884-018-0015-3},
	Volume = {1},
	Year = {2018}}

@inproceedings{bagnell2003covariant,
  author={J. Andrew Bagnell and Jeff G. Schneider},
  title={Covariant Policy Search},
  year={2003},
  cdate={1041379200000},
  pages={1019-1024},
  nourl={http://ijcai.org/Proceedings/03/Papers/146.pdf},
  booktitle={IJCAI}, 
  nocrossref={conf/ijcai/2003}
}

@inproceedings{peters2003reinforcement,
  title={Reinforcement learning for humanoid robotics},
  author={Peters, Jan and Vijayakumar, Sethu and Schaal, Stefan},
  booktitle={Proceedings of the third IEEE-RAS international conference on humanoid robots},
  pages={1--20},
  year={2003}
}

@article{kakade2001natural,
  title={A natural policy gradient},
  author={Kakade, Sham M},
  journal={Advances in Neural Information Processing Systems},
  volume={14},
  year={2001}
}

@article{courte2023robin,
  title={Robin {Pre-Training for the Deep Ritz Method}},
  author={Courte, Luca and Zeinhofer, Marius},
  journal={Northern Lights Deep Learning Conference},
  year={2023}
}

@inproceedings{morimura2008new,
  title={A New Natural Policy Gradient
by Stationary Distribution Metric},
  author={Morimura, Tetsuro and Uchibe, Eiji and Yoshimoto, Junichiro and Doya, Kenji},
  booktitle={Joint European Conference on Machine Learning and Knowledge Discovery in Databases},
  pages={82--97},
  year={2008},
  organization={Springer}
}

@article{martens2020new,
  title={New insights and perspectives on the natural gradient method},
  author={Martens, James},
  journal={The Journal of Machine Learning Research},
  volume={21},
  number={1},
  pages={5776--5851},
  year={2020},
  publisher={JMLRORG}
}

@article{van2022invariance,
  title={Invariance properties of the natural gradient in overparametrised systems},
  author={van Oostrum, Jesse and M{\"u}ller, Johannes and Ay, Nihat},
  journal={Information Geometry},
  pages={1--17},
  year={2022},
  publisher={Springer}
}

@article{amari2010information,
  title={Information geometry of divergence functions},
  author={Amari, {Shun-ichi} and Cichocki, Andrzej},
  journal={Bulletin of the {P}olish academy of sciences. Technical sciences},
  volume={58},
  number={1},
  pages={183--195},
  year={2010}
}

@article{wang2022hessian,
  title={Hessian informed mirror descent},
  author={Wang, Li and Yan, Ming},
  journal={Journal of Scientific Computing},
  volume={92},
  number={3},
  pages={1--22},
  year={2022},
  publisher={Springer}
}

@book{amari2016information,
  title={Information geometry and its applications},
  author={Amari, {Shun-ichi}},
  volume={194},
  year={2016},
  publisher={Springer},
  address={Japan}
}

@article{nurbekyan2022efficient,
  title={Efficient Natural Gradient Descent Methods for Large-Scale Optimization Problems},
  author={Nurbekyan, Levon and Lei, Wanzhou and Yang, Yunan},
  journal={arXiv:2202.06236},
  year={2022}
}

@article{weinan2017deep,
  title={Deep learning-based numerical methods for high-dimensional parabolic partial differential equations and backward stochastic differential equations},
  author={E, Weinan and Han, Jiequn and Jentzen, Arnulf},
  journal={Communications in Mathematics and Statistics},
  volume={5},
  number={4},
  pages={349--380},
  year={2017},
  publisher={Springer}
}

@article{weinan2018deep,
  title={{The Deep Ritz Method: A Deep Learning-Based Numerical Algorithm for Solving Variational Problems}},
  author={E, Weinan and Yu, Bing},
  journal={Communications in Mathematics and Statistics},
  volume={6},
  number={1},
  pages={1--12},
  year={2018},
  publisher={Springer}
}

@article{dissanayake1994neural,
  title={Neural-network-based approximations for solving partial differential equations},
  author={Dissanayake, MWMG and Phan-Thien, N},
  journal={communications in Numerical Methods in Engineering},
  volume={10},
  number={3},
  pages={195--201},
  year={1994},
  publisher={Wiley Online Library}
}

@article{ritz1909neue,
  title={{{\"U}ber eine neue Methode zur L{\"o}sung gewisser Variationsprobleme der mathematischen Physik.}},
  author={Ritz, Walter},
  journal={Journal f{\"u}r die reine und angewandte Mathematik (Crelles Journal)},
  volume={1909},
  number={135},
  pages={1--61},
  year={1909},
  publisher={De Gruyter}
}

@article{lagaris1998artificial,
  title={Artificial neural networks for solving ordinary and partial differential equations},
  author={Lagaris, Isaac E and Likas, Aristidis and Fotiadis, Dimitrios I},
  journal={IEEE transactions on neural networks},
  volume={9},
  number={5},
  pages={987--1000},
  year={1998},
  publisher={IEEE}
}

@inproceedings{jacot2018neural,
  title={{Neural tangent kernel: Convergence and generalization in neural networks}},
  author={Jacot, Arthur and Gabriel, Franck and Hongler, Cl{\'e}ment},
  booktitle={Advances in neural information processing systems},
  pages={8571--8580},
  year={2018}
}

@article{sirignano2018dgm,
  title={{DGM: A deep learning algorithm for solving partial differential equations}},
  author={Sirignano, Justin and Spiliopoulos, Konstantinos},
  journal={Journal of computational physics},
  volume={375},
  pages={1339--1364},
  year={2018},
  publisher={Elsevier}
}

@article{beck2020overview,
  title={An overview on deep learning-based approximation methods for partial differential equations},
  author={Beck, Christian and Hutzenthaler, Martin and Jentzen, Arnulf and Kuckuck, Benno},
  journal={arXiv preprint arXiv:2012.12348},
  year={2020}
}

@article{han2018solving,
  title={Solving high-dimensional partial differential equations using deep learning},
  author={Han, Jiequn and Jentzen, Arnulf and Weinan, E},
  journal={Proceedings of the National Academy of Sciences},
  volume={115},
  number={34},
  pages={8505--8510},
  year={2018},
  publisher={National Acad Sciences}
}

@article{raissi2019physics,
  title={Physics-informed neural networks: A deep learning framework for solving forward and inverse problems involving nonlinear partial differential equations},
  author={Raissi, Maziar and Perdikaris, Paris and Karniadakis, George E},
  journal={Journal of Computational physics},
  volume={378},
  pages={686--707},
  year={2019},
  publisher={Elsevier}
}

@article{wang2021understanding,
  title={Understanding and mitigating gradient flow pathologies in physics-informed neural networks},
  author={Wang, Sifan and Teng, Yujun and Perdikaris, Paris},
  journal={SIAM Journal on Scientific Computing},
  volume={43},
  number={5},
  pages={A3055--A3081},
  year={2021},
  publisher={SIAM}
}

@article{wang2022and,
  title={When and why {PINNs} fail to train: A neural tangent kernel perspective},
  author={Wang, Sifan and Yu, Xinling and Perdikaris, Paris},
  journal={Journal of Computational Physics},
  volume={449},
  pages={110768},
  year={2022},
  publisher={Elsevier}
}

@article{krishnapriyan2021characterizing,
  title={Characterizing possible failure modes in physics-informed neural networks},
  author={Krishnapriyan, Aditi and Gholami, Amir and Zhe, Shandian and Kirby, Robert and Mahoney, Michael W},
  journal={Advances in Neural Information Processing Systems},
  volume={34},
  pages={26548--26560},
  year={2021}
}

@article{davi2022pso,
  title={PSO-PINN: Physics-Informed Neural Networks Trained with Particle Swarm Optimization},
  author={Davi, Caio and Braga-Neto, Ulisses},
  journal={arXiv preprint arXiv:2202.01943},
  year={2022}
}

@article{nabian2021efficient,
  title={Efficient training of physics-informed neural networks via importance sampling},
  author={Nabian, Mohammad Amin and Gladstone, Rini Jasmine and Meidani, Hadi},
  journal={Computer-Aided Civil and Infrastructure Engineering},
  volume={36},
  number={8},
  pages={962--977},
  year={2021},
  publisher={Wiley Online Library}
}

@article{van2022optimally,
  title={Optimally weighted loss functions for solving pdes with neural networks},
  author={van der Meer, Remco and Oosterlee, Cornelis W and Borovykh, Anastasia},
  journal={Journal of Computational and Applied Mathematics},
  volume={405},
  pages={113887},
  year={2022},
  publisher={Elsevier}
}

@article{weinan2021algorithms,
  title={Algorithms for solving high dimensional {PDEs}: from nonlinear Monte Carlo to machine learning},
  author={Weinan, E and Han, Jiequn and Jentzen, Arnulf},
  journal={Nonlinearity},
  volume={35},
  number={1},
  pages={278},
  year={2021},
  publisher={IOP Publishing}
}

@inproceedings{
li2021fourier,
title={Fourier Neural Operator for Parametric Partial Differential Equations},
author={Zongyi Li and Nikola Borislavov Kovachki and Kamyar Azizzadenesheli and Burigede liu and Kaushik Bhattacharya and Andrew Stuart and Anima Anandkumar},
booktitle={International Conference on Learning Representations},
year={2021},
url={https://openreview.net/forum?id=c8P9NQVtmnO}
}

@inproceedings{muller2022notes,
  title={Notes on exact boundary values in residual minimisation},
  author={M{\"u}ller, Johannes and Zeinhofer, Marius},
  booktitle={Mathematical and Scientific Machine Learning},
  pages={231--240},
  year={2022},
  organization={PMLR}
}

@inproceedings{pascanu2014revisiting,
    title={Revisiting natural gradient for deep networks},
    author={Pascanu, Razvan and Bengio, Yoshua},
    booktitle={International Conference on Learning Representations},
    year={2014},
    url={https://openreview.net/forum?id=vz8AumxkAfz5U}
}

@inproceedings{lin2021wasserstein,
  title={Wasserstein proximal of GANs},
  author={Lin, Alex Tong and Li, Wuchen and Osher, Stanley and Mont{\'u}far, Guido},
  booktitle={International Conference on Geometric Science of Information},
  pages={524--533},
  year={2021},
  organization={Springer}
}

@article{schraudolph2002fast,
  title={Fast curvature matrix-vector products for second-order gradient descent},
  author={Schraudolph, Nicol N},
  journal={Neural computation},
  volume={14},
  number={7},
  pages={1723--1738},
  year={2002},
  publisher={MIT Press}
}

@article{shen2020sinkhorn,
  title={Sinkhorn natural gradient for generative models},
  author={Shen, Zebang and Wang, Zhenfu and Ribeiro, Alejandro and Hassani, Hamed},
  journal={Advances in Neural Information Processing Systems},
  volume={33},
  pages={1646--1656},
  year={2020}
}

@article{kovachki2021neural,
  title={Neural operator: Learning maps between function spaces},
  author={Kovachki, Nikola and Li, Zongyi and Liu, Burigede and Azizzadenesheli, Kamyar and Bhattacharya, Kaushik and Stuart, Andrew and Anandkumar, Anima},
  journal={arXiv preprint arXiv:2108.08481},
  year={2021}
}

@article{cai2019gram,
  title={Gram-gauss-newton method: Learning overparameterized neural networks for regression problems},
  author={Cai, Tianle and Gao, Ruiqi and Hou, Jikai and Chen, Siyu and Wang, Dong and He, Di and Zhang, Zhihua and Wang, Liwei},
  journal={arXiv preprint arXiv:1905.11675},
  year={2019}
}

@article{ren2019efficient,
  title={Efficient subsampled Gauss-Newton and natural gradient methods for training neural networks},
  author={Ren, Yi and Goldfarb, Donald},
  journal={arXiv preprint arXiv:1906.02353},
  year={2019}
}

@article{gargiani2020promise,
  title={On the promise of the stochastic generalized Gauss-Newton method for training DNNs},
  author={Gargiani, Matilde and Zanelli, Andrea and Diehl, Moritz and Hutter, Frank},
  journal={arXiv preprint arXiv:2006.02409},
  year={2020}
}

@article{amari1996neural,
  title={Neural learning in structured parameter spaces-natural Riemannian gradient},
  author={Amari, Shun-ichi},
  journal={Advances in neural information processing systems},
  volume={9},
  year={1996}
}

\end{document}